\documentclass{article}
\usepackage[utf8]{inputenc} 
\usepackage[T1]{fontenc}    
\PassOptionsToPackage{hyphens}{url} 
\usepackage[hyperfootnotes=false]{hyperref}
\usepackage{booktabs}       
\usepackage{amsfonts}     
\usepackage{nicefrac}      
\usepackage{microtype}      
\usepackage[table]{xcolor} 
\usepackage{natbib}
\usepackage{breqn}
\usepackage[hang,flushmargin]{footmisc}
\usepackage{algorithm}
\usepackage{algpseudocode}
\usepackage{graphicx}
\usepackage{enumerate}
\usepackage{caption}
\usepackage{tcolorbox}      
\usepackage{array}           
\usepackage{subcaption}
\usepackage{comment}
\usepackage[export]{adjustbox}
\usepackage{mathtools}
\usepackage{amsthm}
\usepackage{authblk}
\usepackage{fullpage}
\usepackage{enumitem}
\usepackage{seqsplit}

\usepackage{tikz}
\usetikzlibrary{positioning}
\usetikzlibrary{calc}
\usetikzlibrary{fit}
\usepackage[dvipsnames]{xcolor}
\usepackage{tcolorbox}

\usepackage{Definitions}
\usepackage{dsfont}
\hypersetup{
    colorlinks=true,
    linkcolor=blue,
    citecolor=blue,
    urlcolor=blue
}

\usepackage{color-edits}
\addauthor{tq}{purple}
\addauthor{aav}{blue}

\usepackage{amssymb}

\def\E{\mathbb{E}}

\def\P{\mathcal{P}}

\def\y{\mathbf{y}}

\def\leader{s_{(1)}}
\def\ru{s_{(2)}}

\usepackage{bbm}
\usepackage{cancel}
\usepackage{nicematrix,tikz}
\usetikzlibrary{arrows.meta, bending, positioning, shapes.geometric}
\usepackage{pgfplots}
\usepackage{pgfplotstable}
\usepgfplotslibrary{dateplot}

\pgfplotsset{compat=1.18}
\usetikzlibrary{plotmarks}
\usepgfplotslibrary{fillbetween}

\newlength{\inlineheight}
\newcommand{\UnnumberedFootnote}[1]{{\def\thefootnote{}\footnote{#1}
\addtocounter{footnote}{-1}}}

\definecolor{mycommentcolor}{RGB}{46,139,87} 
\renewcommand{\Comment}[1]{\hfill\textcolor{mycommentcolor}{\(\triangleright\) #1}}

\title{Strategic Self-Consistency}

\author{Tori Qiu$^{\ast,\dagger}$}
\author{Ander Artola Velasco$^{\ast,\S}$}
\author{Manuel~Gomez-Rodriguez$^{\S}$}

\affil{$^{\dagger}$Carnegie Mellon University, Pittsburgh, USA \\ toriq@andrew.cmu.edu}

\affil{$^{\S}$Max Planck Institute for Software Systems, Kaiserslautern, Germany \\
\{avelasco, manuel\}@mpi-sws.org}

\date{}

\begin{document}

\maketitle

\UnnumberedFootnote{$^{\ast}$Equal contribution. Tori Qiu contributed to this work during an internship at the Max Planck Institute for Software Systems.}

\begin{abstract}
Self-consistency has become a popular technique for enhancing the reasoning abilities of large language models by generating multiple reasoning paths and selecting the final answer through a majority vote.
%
However, because model providers typically charge users in proportion to the number of reasoning paths generated, they have a financial incentive to artificially increase the path count.
%
In this work, we show that an unfaithful provider can exploit this incentive using a simple, efficient algorithm while avoiding detection by an auditor: by generating and strategically reordering additional reasoning paths, the algorithm makes every path appear necessary to reach the majority.
%
To validate our algorithm, we conduct experiments with multiple instruct models from the \texttt{Llama} and \texttt{Qwen} families, as well as reasoning models distilled from \texttt{DeepSeek-R1}, on benchmark datasets spanning mathematics, science, and question answering. 
Our results suggest that the distribution of additional reasoning paths generated by our algorithm is heavy-tailed and that substantial capacity to overcharge remains even under the best possible audit designed to keep the false-positive rate below $\alpha = 0.1$.
\end{abstract}

\section{Introduction}
\label{sec:introduction}
State-of-the-art large language models (LLMs) are commonly accessed as a service. A user sends a query to an LLM provider, receives an answer, and pays proportionally to the compute the provider reports was necessary to generate the answer using the model it serves, as measured by the total number of consumed tokens.
%
However, because the user cannot typically verify that the reported compute was actually necessary, the provider has a financial incentive to inflate or waste compute.

The growing popularity of test-time compute methods, which improve answer quality by spending additional compute at inference time~\citep{wei2022chain,yao2023treeofthoughts,chow2024inference, koh2025tree, yoshiyama2026testtime}, sharpens this incentive. 
Under self-consistency~\citep{wang2023selfconsistency}, one of the most widely used test-time compute methods, the provider generates multiple reasoning paths using the model it serves and selects the final answer through a majority vote. 
In this context, an unfaithful provider can artificially increase the number of reasoning paths and claim that all of them were necessary to determine the answer with high confidence.

One tempting solution is to require the provider both to commit to an adaptive stopping rule that determines the number of reasoning paths required to identify the answer to a user query through a majority vote~\citep{aggarwal2023samplestepbystep, feng2026optimal, huang2026optimal} and to disclose  the entire generated sequence of reasoning paths to the user.
In this work, we show that such a solution fails, even when the reported reasoning paths can be inspected by a powerful third-party auditor.
More specifically, we make the following contributions:
\begin{enumerate}
    \item We introduce a simple and efficient algorithm that an unfaithful provider can use to artificially inflate the number of reasoning paths it generates while making each additional path appear necessary to reach the majority under the provider's stopping rule.
    \item We prove that, under natural conditions on the provider's stopping rule, our algorithm can generate a nontrivial number of additional reasoning paths, and we derive a lower bound on the expected number of additional generations.
    \item We show that an unfaithful provider running our algorithm can strategically evade an exact likelihood-ratio audit, highlighting the vulnerability of users in the current pay-for-compute market.
\end{enumerate}
To complement our theoretical results, we conduct experiments with models from the \texttt{Llama} and \texttt{Qwen} families on mathematics, science, and question-answering benchmarks. 
The results suggest that the distribution of additional reasoning paths generated by our algorithm is heavy-tailed and, even under the best possible audit designed to keep the false positive rate below $\alpha = 0.1$, the provider's capacity to overcharge remains.\footnote{In the main text, we focus on self-consistency. However, in Appendices~\ref{app:beyond-majority} and~\ref{appx:subsec:bon_results}, we empirically extend the analysis to best-of-$N$ and discuss extensions to sequential search-and-verification procedures.} Code for reproducing our experiments is available at \url{https://github.com/Human-Centric-Machine-Learning/strategic-self-consistency}.

\xhdr{Further related work}
Our work builds on a large literature on test-time compute methods~\citep{wang2023selfconsistency,koh2025tree,bi2025forestofthought,karan2026reasoning,wan2025beacon,wu2026when,ramji2026thinking,komiyama2026bestofinfinity,yoshiyama2026testtime}. Most closely related is a line of work on adaptive self-consistency~\citep{aggarwal2023samplestepbystep,feng2026optimal, huang2026optimal}, which develops adaptive stopping rules with statistical guarantees on the final answer. 
In contrast, we demonstrate that an unfaithful provider can create the appearance of following these adaptive stopping rules while artificially inflating the number of reasoning paths it generates.

Our work also connects to a rapidly growing literature on the economic incentives of LLM providers in the market of LLM-as-a-service~\citep{raghavan2024competition,mahmood2024pricing,qiu2025modeling, Bergemann25, LauferFine,olmedo2026computational,chen2026leaderboard, velasco2026overcharging,velasco2026ttcgames,velasco2026auditing}. Within this literature, the works most closely related to consider how unfaithful providers may (i) covertly substitute a  cheaper, lower-quality LLM for the one they charge users for running~\citep{Amirizaniani24,bourree2025robustmlauditing,cai2025gettingpayforauditing,chauvin2026tokenefficient,zhu2026auditing}, or (ii) manipulate the number of tokens billed to users~\citep{wang2025predictiveauditing,sun2025coincountinginvisiblereasoning,velasco2026auditing}. 
To the best of our knowledge, our work is the first to consider unfaithful providers who may strategically increase and reorder additional reasoning paths to overcharge users.

Finally, the audit procedure we study draws on a strand of the change-point detection literature focusing on detecting non-exchangeability~\citep{vovk2003testingexchangeability,pmlr-v152-vovk21b,dandapanthula2026offline,saha2026distribution}. This literature typically relies on conformal or plug-in martingales, whereas in our audit procedure, the likelihood ratio can be computed exactly and efficiently.

\section{Adaptive Self-Consistency with Count-Based Stopping} \label{sec:self_consistency}
Let $\Acal = \{a_i\}_{i \in [K]}$, with $K \in \{2, 3, \ldots\}$, denote the set of possible answers that the provider's model can generate for a user's query,\footnote{In practice, a model generates a reasoning path, which eventually yields an answer $a_i$. Different reasoning paths may represent the same answer with varying wording and style, but explicitly modeling this variation is not relevant to our analysis.}
and let $p_i \geq 0$ be the probability of the provider's model generating answer $a_i$, with $\sum_{i=1}^K p_i = 1$.

A provider seeking to improve the quality of the final answer returned to the user can spend additional compute to generate multiple answers from the model and select the most common one as the final answer---a popular strategy known as self-consistency~\citep{wang2023selfconsistency}.
Concretely, the provider generates a sequence $\y = (y_1, \dots, y_N)$ of independent answers from the model and returns the empirical mode $a_{\widehat{\imath}}$, where
\begin{equation}\label{eq:definition-counts}
    \widehat{\imath} \in \arg\max_{i \in [K]} s_i \quad \text{and} \quad s_i = \sum_{j=1}^N \mathbbm{1}[y_j = a_i].
\end{equation}
Thus, self-consistency uses $N$ samples to  estimate the mode of the model's answer distribution---that is, its most likely answer $a_{i^\star}$, where $i^\star \in \arg\max_{i \in [K]} p_i$. Throughout, we assume that this mode is unique.
While increasing $N$ yields a more reliable estimate of $a_{i^\star}$, it also incurs greater computational cost. This tradeoff has motivated interest in adaptive stopping rules~\citep{aggarwal2023samplestepbystep, feng2026optimal, huang2026optimal}, which sequentially evaluate generated answers to decide whether enough answers have been generated to identify the mode with high confidence. Rather than fixing $N$ in advance, these rules aim to generate only as many answers as a query requires.

For the remainder of the main analysis, we consider a setting where the provider commits to an adaptive stopping rule $\tau(s_1, \ldots, s_K) \in \{0, 1\}$, where $1$ indicates that  the counts $s_1, \ldots, s_K$ are sufficient to stop and identify the mode with high confidence and $0$ indicates that generation should continue.  The provider  discloses to the user the entire sequence of generated answers $\y$, together with the empirical mode $a_{\widehat{\imath}}$.
We focus on stopping rules that depend on the answer counts $(s_1, \dots, s_K)$ rather than on the full sequence $\y$, since these counts are sufficient statistics for the answer probabilities $(p_1, \dots, p_K)$, and consequently, for the mode of the distribution. 
We illustrate two such stopping rules below and discuss additional rules from the literature in Appendix~\ref{appx:stopping_rules}.
\begin{example}[PPR-1v1 stopping rule~\citep{pac_mode_estimation}] The PPR-1v1 stopping rule constructs a prior-posterior-ratio (PPR) martingale confidence sequence for the probability $p_{\widehat{\imath}}$ of the empirical mode after $N$ observed answers and stops as soon as this confidence sequence only contains sufficiently high values, declaring $a_{\widehat{\imath}}$ to be the mode. 
Concretely, the PPR-1v1 stopping rule can be written as
\begin{align*}
    \tau_{\text{PPR-1v1}}(s_1, \dots, s_K) = \mathbbm{1} \left[ \frac{(1/2)^{s_{(1)}}\cdot \, (1/2)^{s_{(2)}}}{\int_{0}^{1} q^{s_{(1)}} \cdot(1-q)^{s_{(2)}} \, dq} \leq \frac{\delta}{K-1} \right],
\end{align*}
where $\delta \in (0, 1)$ is a user-defined bound on the probability that the empirical mode $a_{\widehat{\imath}}$ does not coincide with the true mode $a_{i^\star}$ at stopping, and $s_{(1)}\geq s_{(2)}\geq \dots \geq s_{(K)}$ are the ordered answer counts.
\label{ex:ppr-1v1}
\end{example}
While $\tau_{\text{PPR-1v1}}$ provides theoretical guarantees on the probability of incorrectly identifying the mode, it can be overly conservative under practical answer-generation budgets~\citep{feng2026optimal}. Other stopping rules therefore relax these theoretical guarantees in favor of stronger empirical performance.

\begin{example}[ASC stopping rule~\citep{aggarwal2023samplestepbystep}] The ASC stopping rule approximates the posterior probability that the current empirical mode is the true mode using a uniform prior over the possible answer probabilities $(p_1, \dots, p_K)$. Specifically, for a fixed parameter $\gamma \in (0,1)$, the stopping rule is:
\begin{align*}
    \tau_{\text{ASC}}(s_1, \dots, s_K) &= \mathbbm{1} \left[ \frac{\int_{1/2}^{1} q^{s_{(1)}} \cdot (1-q)^{s_{(2)}} \, dq}{\int_{0}^{1} q^{s_{(1)}}\cdot(1-q)^{s_{(2)}} \, dq} \ge \gamma \right].
\end{align*}
Unlike $\tau_{\text{PPR-1v1}}$, $\gamma$ does not necessarily bound the probability that the empirical mode is incorrect.
\label{ex:ASC}
\end{example}
A provider that faithfully implements a stopping rule generates exactly as many answers as the rule prescribes. In the next section, however, we show that an unfaithful provider can nonetheless artificially increase the number of generated answers and overcharge the user while making each additional answer appear necessary under the stopping rule.

\section{Inflating Reasoning Paths Under Adaptive Stopping}
\label{sec:protocol}
Suppose a provider commits to an adaptive stopping rule $\tau$ that prescribes stopping after $N$ generated answers on a given query. A faithful provider would report the resulting sequence $\y=(y_1,\ldots,y_N)$ and charge the user for exactly these $N$ generations.
However, an unfaithful provider has an incentive to continue generating additional answers; if it reports a longer sequence $\y' = (y_1, \ldots, y_{N'})$ with $N' > N$, they can charge the user for the additional $N' - N$ generations.

\begin{algorithm}[!t]
\footnotesize
\algrenewcommand\algorithmicprocedure{\textbf{function}}

\caption{It returns a longer compatible sequence of answers}\label{alg:greedy_lookahead}
\begin{algorithmic}[1]
\renewcommand{\algorithmicrequire}{\textbf{Input: }}
\renewcommand{\algorithmicensure}{\textbf{Output: }}

\Require{Stopping rule $\tau$, compatible sequence $\y = (y_1, \ldots, y_N)$, budget $N_{\max}$, number of answers $K$, audit significance level $\alpha$}
\Ensure{Extended sequence $\y'$}

\vspace{1mm}

\State $\y' \gets \y$
\State $\texttt{LastCompatible} \gets \y$
\While{$\mathrm{len}(\y') < N_{\max}$}
    \State $m \gets \mathrm{len}(\y')$
    \State $(s'_1, \ldots, s'_K) \gets$ answer counts of $\y'$ 
    \State $y_{\mathrm{new}} \gets$ answer of a new reasoning path generated by the model \Comment{$\Pr(y_{\mathrm{new}} = a_i) = p_i$}
    \If{$\tau(s'_1,\dots,s'_K) = 0$} \Comment{The current sequence does not trigger $\tau$}
        \State $\y''= (y_1,\dots, y_m, y_{\mathrm{new}})$
    \Else
        \State $\texttt{LastCompatible} \gets \y'$
        \If{$y_{\mathrm{new}} = y_m$}
            \State \textbf{return} $\y'$ \Comment{Cannot defer: the new answer matches the current last answer} \label{line:discard}
        \EndIf
        \State $\y'' \gets (y_1, \ldots, y_{m-1}, y_{\mathrm{new}}, y_m)$ \Comment{Insert $y_{\mathrm{new}}$ before the last answer}
        \State $(s_1, \ldots, s_K) \gets$ answer counts of $ (y_1, \ldots, y_{m-1}, y_{\mathrm{new}})$ 
        \If{$\tau(s_1, \ldots, s_K) = 1$}
            \State \textbf{return} $\y'$ \Comment{Cannot defer: the new answer triggers the stopping rule}
        \EndIf
    \EndIf
    \If{$\texttt{AuditFlag}(\y'', \alpha)$} \Comment{$\y''$ exceeds audit threshold (Eq.~\eqref{eq:likelihood-test})} \label{line:start-flag}
        \State \textbf{return} \texttt{LastCompatible}
    \EndIf \label{line:end-flag}
    
    \State $\y' \gets \y''$ \Comment{Continue with the longer, still non-triggering sequence}
\EndWhile
\State \textbf{return} $\y'$
\end{algorithmic}
\end{algorithm}
The provider cannot simply report an arbitrary longer sequence $\y'$ without raising the user's suspicion. Since the user knows the stopping rule $\tau$ and observes the entire reported sequence $\y'$, they can verify whether the provider stopped generating answers when prescribed by the rule. In particular, if any prefix of the reported sequence already triggers the stopping rule, the user can conclude that the provider generated unnecessary answers. A provider seeking to avoid  suspicion is therefore restricted to reporting sequences that are \emph{compatible} with the stopping rule.

\begin{definition}[Compatible sequence]\label{def:compatible-sequence}
A sequence of answers $\y = (y_1, \ldots, y_N)$ is \emph{compatible} with a stopping rule $\tau$ if
\begin{equation}
    \tau(s_{1}^n, \dots, s_{K}^n) = 0 \ \text{ for all } n < N, \qquad \text{and} \qquad \tau(s_{1}^{N}, \dots, s_{K}^{N}) = 1,
\end{equation}
where $s_{1}^n, \dots, s_{K}^n$ denote the answer counts among the first $n$ elements of $\y$.
\end{definition}
The definition above captures the intuition that, from the user’s perspective, every answer in a compatible sequence appears necessary to identify the most likely answer with confidence, as prescribed by the stopping rule $\tau$.
We next show that an unfaithful provider can artificially increase the number of answers it generates while still reporting a compatible sequence. 
To this end, we introduce a simple, heuristic procedure (Algorithm~\ref{alg:greedy_lookahead}) that transforms a sequence $\y$ of length $N$, generated by faithfully following stopping rule $\tau$, into a longer compatible sequence $\y'$ of length $N' \geq N$, where the additional $N' - N$ answers are genuinely generated by the model.

Algorithm~\ref{alg:greedy_lookahead} starts from the faithfully stopped sequence $\y$ and extends it one answer at a time. Specifically, whenever the current sequence $(y_1, \ldots, y_{m-1}, y_m)$ triggers the stopping rule---and would therefore require a faithful provider to stop---the algorithm generates one additional answer $y_{\mathrm{new}}$ from the model and checks whether the sequence $(y_1, \ldots, y_{m-1}, y_{\mathrm{new}})$, obtained by replacing the last answer $y_m$ with $y_{\mathrm{new}}$, triggers the stopping rule. If it does not, the algorithm continues from the reordered sequence $(y_1, \ldots, y_{m-1}, y_{\mathrm{new}}, y_m)$; otherwise, it discards $y_{\mathrm{new}}$ and terminates, returning the most recent compatible sequence.\footnote{In principle, more sophisticated strategies could reorder the entire sequence, rather than only the final pair of answers. We focus on Algorithm~\ref{alg:greedy_lookahead} to highlight that even a simple heuristic suffices to artificially increase the number of answers.} Importantly, Algorithm~\ref{alg:greedy_lookahead} guarantees by construction that any sequence it returns appears, from the user's perspective, to have stopped exactly as prescribed by $\tau$. We formalize this property in the following proposition (see Appendix~\ref{appx:attack_proofs} for the proof).
\begin{proposition}\label{prop:one_step_correctness}
The sequence $\y'$ returned by Algorithm~\ref{alg:greedy_lookahead} is compatible with the stopping rule $\tau$.
\end{proposition}

%
%
The proposition above guarantees that an unfaithful provider using Algorithm~\ref{alg:greedy_lookahead} can make every additional answer appear necessary under the stopping rule, but it does not quantify the benefit that the provider can obtain from this strategy.
To address this, we derive a lower bound on the average number of additional answers in the sequence  $\y'$ reported to the user for a broad class of stopping rules, which we refer to as admissible stopping rules. This class includes PPR-1v1 (Example~\ref{ex:ppr-1v1}) and ASC (Example~\ref{ex:ASC}).
\begin{definition}[Admissible stopping rule]\label{def:admissible_count_based_stopping}
A stopping rule $\tau$ is \emph{admissible} if it satisfies the following properties.
\begin{enumerate}[label=(\roman*)]
    \item \textbf{Count invariance:} $\tau(s_1, s_2, \dots, s_K)$ depends only on the counts of the first and second most frequent answers, which we write as $\tau(s_{(1)}, s_{(2)})$.

    \item \textbf{Nontriviality:} a single observed answer is insufficient to identify the mode with high confidence; that is, if $s_{(1)}=1$ and $s_{(2)}=0$, then $\tau(s_{(1)}, s_{(2)}) = 0$.

    \item \textbf{Monotonicity in the second most frequent count:} the stopping rule is non-increasing in the second most frequent count; that is, if $\tau(s_{(1)}, s_{(2)}) = 0$ and $s_{(1)} \geq s_{(2)} + 1$, then $\tau(s_{(1)}, s_{(2)} + 1) = 0$.

    \item \textbf{Fixed-margin monotonicity:} if $\tau(s_{(1)}, s_{(2)}) = 0$, then $\tau(s_{(1)} + c, s_{(2)} + c) = 0$ for all $c \in \mathbb{N}^+$.
\end{enumerate}
\end{definition}
Condition $(i)$ states that, to determine whether the current empirical mode can be reliably identified as the true mode, only the counts of the two most frequent answers matter, rather than those of the remaining answers or the order in which the answers were generated. Condition $(ii)$ rules out trivial stopping rules. Finally, conditions $(iii)$ and $(iv)$ capture the intuition that the provider's confidence in the empirical mode cannot increase either by observing an additional second most frequent answer, which brings its count closer to that of the mode, or by observing the empirical mode and the second most frequent answer in equal measure, which leaves their count difference unchanged. 
The following result lower-bounds the expected number of additional answers obtained from running Algorithm~\ref{alg:greedy_lookahead}. We focus here on the case of two answers, $K = 2$, and refer the reader to Appendix~\ref{appx:subsec:lower_bound_gain} for the general case.
\begin{proposition}\label{prop:provable_gain_binary} 
Let $\y=(y_1,\dots,y_N)$ be any sequence of answers compatible with an admissible stopping rule $\tau$, and let $\y'=(y'_1,\dots, y'_{N'})$ be a sequence returned by Algorithm~\ref{alg:greedy_lookahead}. If $K=2$, then it holds that:
\begin{equation}\label{eq:lower-bound-main}
    \mathbb{E}[N' - N] \geq \frac{p_{(1)} p_{(2)}\,(p_{(1)}^{\,d} - p_{(2)}^{\,d}) - d\,(p_{(1)} - p_{(2)})\, p_{(2)}^{\,d+1}}{(p_{(1)} - p_{(2)})\,(p_{(1)}^{\,d+2} - p_{(2)}^{\,d+2})} \quad \text{ where }\quad  d = s_{(1)} - s_{(2)} - 1,
\end{equation}
where $(s_{(1)}, s_{(2)})$ denote the answer counts in $\y$ in non-increasing order, $(p_{(1)}, p_{(2)})$ denote the corresponding answer probabilities under the model, and the expectation is taken over the additional answers generated by the model within Algorithm~\ref{alg:greedy_lookahead}.
\end{proposition}
The intuition for the bound, formalized in Appendix~\ref{appx:subsec:lower_bound_gain}, follows from the monotonicity properties of admissible stopping rules. Sampling additional occurrences of the second most frequent answer $a_{(2)}$ in the sequence $\y$ requires the provider to generate at least as many occurrences of the most frequent answer $a_{(1)}$ before the stopping rule can trigger again. Algorithm~\ref{alg:greedy_lookahead} can therefore continue extending the sequence as it generates additional occurrences of $a_{(2)}$. Accordingly, the lower bound grows as $p_{(2)} \to p_{(1)}$ and the mode of the answer distribution becomes harder to distinguish. This is more likely for difficult queries, where the model is less confident in its answer. In contrast, easier queries have answer distributions more concentrated around the mode, leaving less uncertainty for the provider to exploit by generating additional answers. We verify this relationship empirically in Appendix~\ref{app:synthetic-data}.
Lastly,  while the bound in Proposition~\ref{prop:provable_gain_binary} depends only on the two most frequent answers in the sequence $\y$, Appendix~\ref{appx:subsec:lower_bound_gain} derives a tighter lower bound on $\E[N'-N]$ that depends on the counts and probabilities of all $K$ answers, which can be computed by solving a linear system of equations.

\section{Auditing Self-Consistency Reasoning Paths}
\label{sec:auditing_individual_seqs}
We showed in the previous section that an unfaithful provider using Algorithm~\ref{alg:greedy_lookahead} can generate additional self-consistency reasoning paths that appear necessary under the stopping rule. 
In this section, we show how an unfaithful provider can leverage an exact likelihood-ratio test to ensure the sequence generated by Algorithm~\ref{alg:greedy_lookahead} remains statistically indistinguishable from one generated faithfully, thereby avoiding detection by an auditor.

To make this precise, we formalize the auditor's inspection of a reported sequence $\y'$ as a hypothesis test with the following null and alternative hypotheses.
\begin{equation*}
    \begin{dcases}
        H_0:  \y' \overset{ \mathrm{iid}}{\sim} \mathrm{Categorical}(p_1, \ldots, p_K) & \quad (\mathrm{null})\\
         H_1: \y' \,\, \text{is generated by Algorithm~\ref{alg:greedy_lookahead} from} \,\, \y \overset{ \mathrm{iid}}{\sim} \mathrm{Categorical}(p_1, \ldots, p_K) & \quad (\mathrm{alternative}).
    \end{dcases}
\end{equation*}
The null hypothesis $H_0$ corresponds to a faithful provider who generates a sequence following the stopping rule, whereas the alternative $H_1$ corresponds to an unfaithful provider who first generates a sequence $\y$ following the stopping rule and then applies Algorithm~\ref{alg:greedy_lookahead} to obtain a longer sequence $\y'$.

An auditor could use any possible test to reject $H_0$, but the celebrated Neyman--Pearson lemma guarantees that, among all tests of a given significance level $\alpha \in (0, 1)$, the most powerful is the likelihood-ratio test, which rejects $H_0$ whenever the likelihood ratio $\mathbb{P}_{H_1}(\y')/\mathbb{P}_{H_0}(\y')$ exceeds a  threshold $c_\alpha$, chosen so that the test has level $\alpha$~\citep{Lehmann2005}. Computing $c_\alpha$ exactly is typically intractable because it requires knowledge of the distribution of the likelihood ratio under $H_0$. However, since the likelihood ratio is an \emph{e-value}, setting the threshold to the explicit value $1/\alpha$ guarantees that the probability of falsely flagging a faithful provider is at most $\alpha$~\citep{grunwald2024, ramdas2025evalues, velasco2026auditing}, resulting in the following hypothesis test:\footnote{For ease of exposition, we focus on a setting where the provider reports a single sequence $\y'$. In Appendix~\ref{appx:subsec:auditing_multiple_seqs}, we extend our analysis to a setting where the provider reports multiple sequences.}
\begin{equation}\label{eq:likelihood-test}
     \text{reject } H_0 \quad \text{if} \quad \frac{\mathbb{P}_{H_1}(\y')}{\mathbb{P}_{H_0}(\y')} \geq \frac{1}{\alpha}.
\end{equation}
%
%

We consider an unfaithful provider who anticipates such an audit and strategically avoids detection by the likelihood-ratio test in Eq.~\eqref{eq:likelihood-test}. 
To do so, the provider must compute the likelihood ratio, which we now show can be done efficiently. The denominator is immediate: under $H_0$, the answers are sampled independently from the model, yielding $\mathbb{P}_{H_0}(\y') = \prod_{i=1}^{N'} p_{y'_i}$.
Computing the numerator $\mathbb{P}_{H_1}(\y')$ is more challenging because several distinct executions of Algorithm~\ref{alg:greedy_lookahead} can return the same reported sequence $\y'$, as the following example illustrates.
\begin{example}
    Consider a query with three possible answers $\mathcal{A} = \{a_1, a_2, a_3 \}$, a sequence $\y' = (a_1, a_2, a_3, a_1)$ reported by the provider, and a stopping rule that instructs the provider to stop  as soon as any answer is generated twice, that is, $\tau(s_{(1)}, s_{(2)}) = \mathbbm{1}\{ s_{(1)} \geq 2 \}$.
    %
    To compute the likelihood $\mathbb{P}_{H_1}(\y')$, observe that Algorithm~\ref{alg:greedy_lookahead} must have been executed on one of the following compatible input sequences $\y$:
    \begin{enumerate}
        \item \textbf{$\y=(a_1,a_2,a_3,a_1)$.} The algorithm then generates one additional $a_1$ and discards it, returning $\y’ = \y$. This occurs with probability $p_1\cdot p_2\cdot p_3\cdot p_1\cdot p_1$.
        \item \textbf{$\y=(a_1,a_2,a_1)$.} The algorithm must then generate $a_3$ and construct the partial sequence $(a_1, a_2, a_3, a_1)$. It subsequently generates one additional $a_1$ and discards it. This occurs with probability $p_1\cdot p_2\cdot p_1\cdot p_3\cdot p_1$.
        \item \textbf{$\y=(a_1,a_1)$.} The algorithm first generates $a_2$ and swaps it with the second occurrence of $a_1$ to obtain $(a_1, a_2,a_1)$, and then generates $a_3$ to construct the sequence $(a_1, a_2 ,a_3, a_1)$. Finally, it generates one additional $a_1$ and discards it. This occurs with probability $p_1\cdot p_1\cdot p_2\cdot p_3\cdot p_1$.    
\end{enumerate}
    Since these cases are mutually exclusive and exhaust all the ways in which Algorithm~\ref{alg:greedy_lookahead} can return $\y'$, summing their probabilities gives $\mathbb{P}_{H_1}(\y')=3\cdot p_1^3 \cdot p_2\cdot p_3$.
\label{ex:path_counts}
\end{example}

\begin{algorithm}[t]
\caption{It computes the likelihood of an observed sequence under $H_1$}
\label{alg:likelihood}

\footnotesize

\begin{algorithmic}[1]

\renewcommand{\algorithmicrequire}{\textbf{Input: }}
\renewcommand{\algorithmicensure}{\textbf{Output: }}

\Require{Compatible sequence $\y' = (y'_1,\dots,y'_{N'})$, stopping rule $\tau$, class probabilities $p_1,\dots,p_K$.}
\Ensure{Likelihood of $\y'$ under Algorithm \ref{alg:greedy_lookahead}}

\vspace{1mm}

\State $\texttt{Counts} \gets \underbrace{(1,0,\dots,0 )}_{N'  \text{ elements}}$

\State 

\For{$m=2,\dots, N' $}
    \State $\texttt{PrevCounts} = 0$ 

    \For{$L = 0, \dots, m - 2$}
        \State \texttt{PossibleMultipleSwap} = \texttt{True}
        \For{$k = 1,\dots,L$}
            \State $(s_1, \ldots, s_K) \gets$ answer counts of $ (y'_1, \ldots, y'_{m - k - 1}, y'_m)$ 
            \If{ $\tau \bigl( s_1, \ldots, s_K \bigr) = 0$ }
            \State \texttt{PossibleMultipleSwap} = \texttt{False} 
            \EndIf
        \EndFor
        \If{\texttt{PossibleMultipleSwap}}
         \State $\texttt{PrevCounts} = \texttt{PrevCounts} + \texttt{Counts}[m - L - 1]$ \label{line:DP-update}
         \EndIf
        
    \EndFor
    \State $\texttt{Counts}[m] = \texttt{PrevCounts}$ 
\EndFor
\vspace{1mm}
\State \Return $\texttt{Counts}[N']\cdot p_{y'_{N'}} \cdot \prod_{i=1}^{N'} p_{y'_i}$ \label{alg-line:likelihood-result}
\end{algorithmic}
\end{algorithm}

The intuition from the example above generalizes. In particular, the provider can efficiently compute  $\mathbb{P}_{H_1}(\y')$ for any compatible sequence using the dynamic programming procedure in Algorithm~\ref{alg:likelihood}. Algorithm~\ref{alg:likelihood} exploits the observation that every answer in a reported sequence $\y' = (y'_1, \dots, y'_{N'})$ returned by Algorithm~\ref{alg:greedy_lookahead} is generated by the model. Each execution generates one additional answer---equal to the last reported answer $y'_{N'}$---which triggers the stopping rule and is discarded. Hence, the probability of obtaining $\y'$ from any given compatible input $\y$ equals $\prod_{i=1}^{N'} p_{y'_i} \cdot p_{y'_{N'}} = \mathbb{P}_{H_0}(\y') \cdot p_{y'_{N'}} $, independent of $\y$.
Computing $\mathbb{P}_{H_1}(\y')$ therefore reduces to counting the number of compatible input sequences $\mathbf{y}$ that can produce $\mathbf{y}'$. 
Algorithm~\ref{alg:likelihood} tracks this quantity iteratively in a vector $\texttt{Counts}$, whose $m$-th entry stores the number of compatible input sequences capable of generating the prefix $(y'_1, \ldots, y'_m)$. The final entry, $\texttt{Counts}[N']$, gives the total count of such sequences. The following proposition formalizes the correctness of this procedure.\footnote{In Appendix~\ref{appx:sec:likelihoods_linear}, we further show that Algorithm~\ref{alg:likelihood} can be implemented in time linear in the length $N'$ of the reported sequence.}
\begin{proposition}\label{prop:likelihood_correctness}
    Given an observed sequence of answers $\y'$ reported by the provider, Algorithm~\ref{alg:likelihood} returns the likelihood $\mathbb{P}_{H_1}(\y')$.
\end{proposition}
As an immediate consequence, the provider can efficiently compute the likelihood ratio in Eq.~\eqref{eq:likelihood-test} using the following expression:\footnote{In practice, the provider can efficiently estimate the probability of the most frequent answer $y'_{N'}$ using the empirical frequency of $y'_{N'}$ among the answers in the generated sequence $\y$.} 
\begin{equation*}
    \frac{\mathbb{P}_{H_1}(\y')}{\mathbb{P}_{H_0}(\y')} = \frac{ \texttt{Counts}[N'] \cdot p_{y'_{N'}}\cdot \mathbb{P}_{H_0}(\y') }{\mathbb{P}_{H_0}(\y')} =  \texttt{Counts}[N'] \cdot p_{y'_{N'}}.
\end{equation*}
To avoid detection, the provider can monitor this ratio while generating additional answers in Algorithm~\ref{alg:greedy_lookahead} and terminate immediately before the value reaches or exceeds the threshold $1/\alpha$.
This is implemented in lines~\ref{line:start-flag}--\ref{line:end-flag} using the indicator:
\begin{equation*}
\texttt{AuditFlag}(\y, \alpha) := \mathbbm{1} \left[\texttt{Counts}[N'] \cdot p_{y_{N'}}\ge \frac{1}{\alpha} \right].
\end{equation*}

\section{Experiments}
\label{sec:experiments}
In this section, we evaluate the empirical performance of Algorithm~\ref{alg:greedy_lookahead} on real queries from popular benchmark datasets across multiple LLM families.
We focus on the ASC stopping rule (Example~\ref{ex:ASC}), which has been shown to outperform PPR-1v1 (Example~\ref{ex:ppr-1v1}) empirically under a variety of conditions~\citep{feng2026optimal}. 
In Appendix~\ref{appx:sec:additional_results}, we show that our conclusions extend to other stopping rules, including potentially inadmissible ones.

\xhdr{Experimental setup}
We consider models from the \texttt{Llama} and \texttt{Qwen} families, including reasoning models distilled from \texttt{DeepSeek-R1}, and queries from  \texttt{GSM8K}~\citep{cobbe2021training} and \texttt{AIME}~\citep{aime25}  for mathematical reasoning and from \texttt{GPQA}~\citep{rein2023gpqagraduatelevelgoogleproofqa} for question answering. 
Following prior work on self-consistency~\citep{feng2026optimal}, we first construct an empirical answer distribution per query using $128$ and $32$ pre-generated answers for the instruct and reasoning models, respectively. 
For each query, we first process the pre-generated answers in their original order until the ASC stopping rule triggers. If the available pre-generated answers are insufficient to reach faithful stopping, we continue by sampling additional answers with replacement from the empirical answer distribution until the stopping rule is satisfied, yielding  $\y = (y_1, \dots, y_N)$.
Finally, starting from each sequence $\y$, we use Algorithm~\ref{alg:greedy_lookahead} to generate the reported sequence $\y'=(y'_1,\ldots,y'_{N'})$, sampling the additional $N'-N$ answers with replacement from the same distribution.
In all experiments, to limit computational overhead, we impose a cap of $5,000$ answers on both $N$ and $N' - N$.

\xhdr{Results} Table~\ref{table:asc_additional_answers_stats} and Figure~\ref{fig:excess_samples_by_difficulty} present summary statistics and the distribution of additional answers $N' - N$ generated by Algorithm~\ref{alg:greedy_lookahead} with $\alpha = 0.1$.
Across models and benchmark datasets, the results suggest that the distribution of additional answers is heavy-tailed: most reported sequences contain fewer than $5$ additional answers, while a long tail of sequences contains hundreds, or even thousands, of additional answers (Table~\ref{table:asc_additional_answers_stats}).
Moreover, the provider's capacity to overcharge varies substantially across benchmark datasets.
Figure~\ref{fig:asc_auditing} shows the $90$th percentile of the distribution of additional answers $N' - N$ for different values of the audit threshold $\alpha$.
Perhaps surprisingly, the provider's capacity to overcharge remains relatively insensitive to the choice of $\alpha$.

\begin{table}[t]
\small
\centering
\begin{tabular}{l|cccc|cccc|cccc}
\toprule
& \multicolumn{4}{c|}{\texttt{Llama-3.2-3B}}
& \multicolumn{4}{c|}{\texttt{Qwen2.5-7B}}
& \multicolumn{4}{c}{\texttt{DeepSeek-R1-Distill-Qwen-7B}} \\
& Mean & Median  & P75 & P90 
& Mean & Median  & P75 &  P90
& Mean & Median  & P75 & P90 \\
\midrule
\texttt{GSM8K} & 13.5 & 0.0 & 0.0 & 5.0 
 & 6.0 &  0.0 &  0.0 & 0.0 
& 3.1  & 0.0 & 0.0  & 3.0 \\
\texttt{AIME} & 101.0 & 5.0  & 17.0 & 63.0
& 82.0 & 3.0  & 11.0  & 56.0
& 69.6 &  3.0  & 11.0 & 39.9 \\
\texttt{GPQA} & 22.8 & 0.0 & 3.0  & 17.0
& 18.2 & 0.0  & 2.0 & 7.0
& 1.0 & 0.0  & 0.0  & 0.0 \\
\bottomrule
\end{tabular}
\caption{\textbf{Summary statistics of the additional answers generated by an unfaithful provider.} For each model and dataset, we report the mean, median, 75th percentile (P75), and 90th percentile (P90) of the additional answers $N'-N$ generated by Algorithm~\ref{alg:greedy_lookahead} on queries from \texttt{GSM8K}, \texttt{AIME}, and \texttt{GPQA}. 
In all experiments, we use the ASC stopping rule (Example~\ref{ex:ASC}) with $\gamma = 0.95$ and set $\alpha =0.1$.}
\label{table:asc_additional_answers_stats}
\end{table}

\begin{figure}[t]
\includegraphics[width=\textwidth]{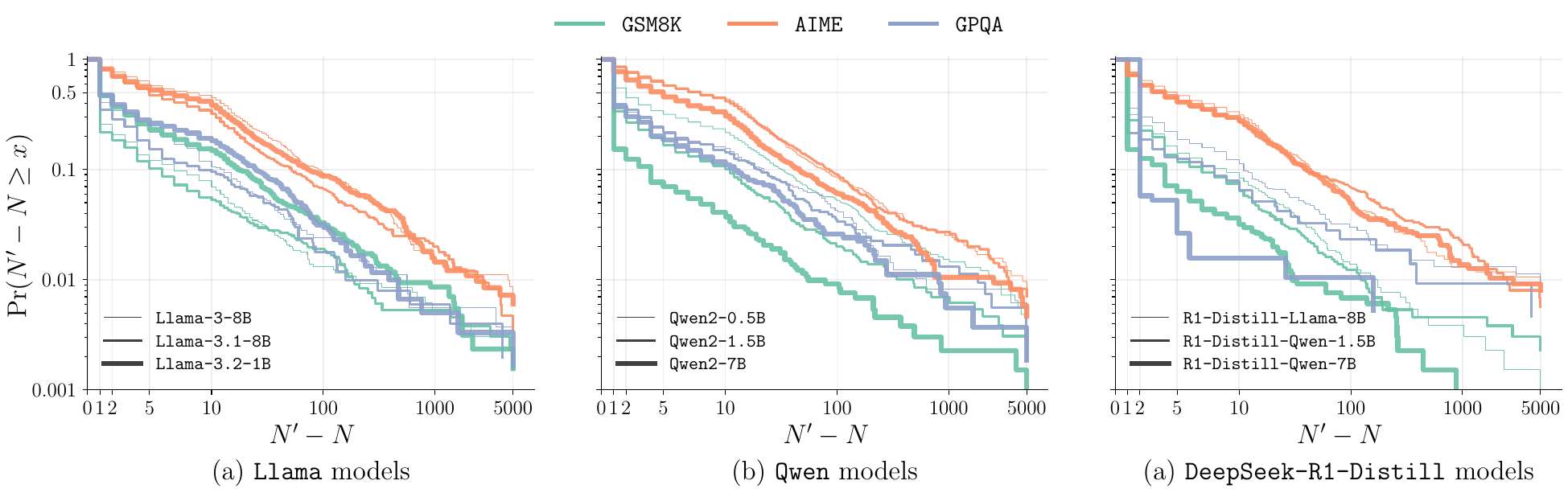}
\caption{\textbf{Distribution of additional answers generated by an unfaithful provider.}  
For queries from \texttt{GSM8K}, \texttt{AIME}, and \texttt{GPQA}, the figure shows the complementary cumulative distribution $\Pr(N'-N \ge x)$ of additional answers $N'-N$ generated by Algorithm~\ref{alg:greedy_lookahead}.
In all experiments, we use the ASC stopping rule (Example~\ref{ex:ASC}) with $\gamma = 0.95$ and set $\alpha = 0.1$.}
\label{fig:excess_samples_by_difficulty}
\end{figure}

\begin{figure}[t]
\centering
\includegraphics[width=\textwidth]{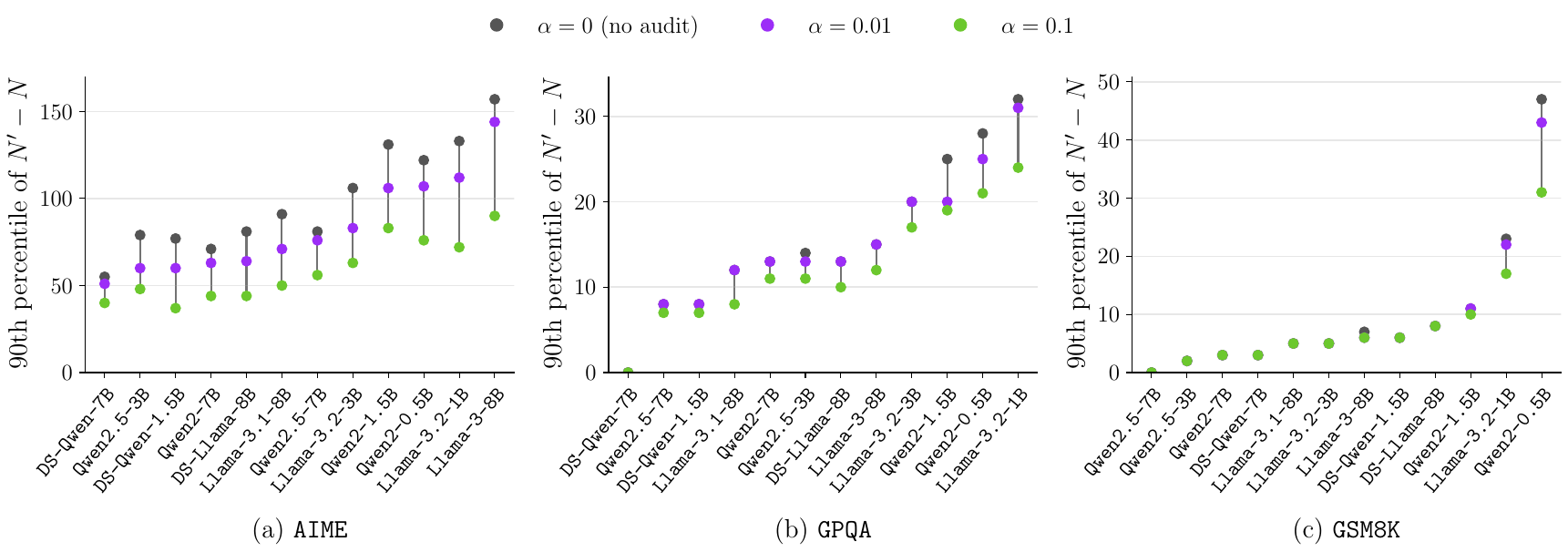}
\caption{\textbf{Influence of the audit threshold on an unfaithful provider's capacity to overcharge.}
For queries from \texttt{GSM8K}, \texttt{AIME}, and \texttt{GPQA}, the figure shows the $90$th percentile of the distribution of additional answers $N'-N$ generated by Algorithm~\ref{alg:greedy_lookahead} for different values of the audit threshold $\alpha$.
In all experiments, we use the ASC stopping rule (Example~\ref{ex:ASC}) with $\gamma = 0.95$.
\texttt{DS} abbreviates \texttt{DeepSeek-R1-Distill}.}
\label{fig:asc_auditing}
\end{figure}

\section{Discussion and Limitations}
\label{sec:discussion}
We discuss several assumptions and limitations of our work and propose avenues for future work.

\xhdr{Methodology}
We have focused on self-consistency, which is widely used for tasks with verifiable, categorical responses. This setting admits a clean statistical interpretation, as stopping rules such as PPR-1v1 offer theoretical guarantees that lend themselves to the analysis in Section~\ref{sec:self_consistency}. In Appendices~\ref{app:beyond-majority} and~\ref{appx:subsec:bon_results}, we discuss how our conclusions transfer to adaptive best-of-$N$ and other test-time compute methods, where stopping rules depend on reward-model scores or other signals. Such stopping rules introduce additional challenges because they typically depend on the correctness of a reward function and assumptions about the reward distribution~\citep{wan2025beacon}.  Developing comparable theoretical bounds on overcharging, as well as  the corresponding audit mechanisms, is an important direction for future work. 

Furthermore, our lower bound on the provider's gain in Proposition~\ref{prop:provable_gain_binary} applies only to admissible stopping rules. This class includes several popular variants, but certain inadmissible stopping rules  nonetheless perform well empirically.
In Appendix~\ref{appx:subsec:inadmissible_stopping_experiments}, we demonstrate that our main conclusions extend to such rules.

\xhdr{Experimental evaluation}
Our experiments cover mathematical reasoning, science, and question-answering tasks. 
It would be valuable to evaluate additional domains, as well as extensions of self-consistency beyond categorical answers. For instance, in coding tasks, agreement is determined by voting on the outputs of generated programs over a set of test cases~\citep{aggarwal2023samplestepbystep}.
Our experiments evaluate ASC in the main text, PPR-1v1 in Appendix~\ref{appx:subsec:synthetic_generation}, and early-stopping self-consistency (ESC)~\citep{li2024escape} in Appendix~\ref{appx:sec:additional_results} as stopping rules. 
Future work could also consider confidence-weighted stopping rules, such as CISC~\citep{taubenfeld-etal-2025-confidence}, and more complex multi-agent test-time compute procedures, which spawn several models and orchestrate their behavior before aggregating their conclusions.\footnote{\url{https://docs.x.ai/developers/model-capabilities/text/multi-agent}.}

\xhdr{Auditing self-consistency}
While we have shown that a provider can evade a powerful single-sequence audit, repeated overcharging may be harder to conceal. An audit can be deployed online, with evidence aggregated across multiple user queries, thereby constraining how frequently or aggressively the provider can misreport the number of reasoning paths (see Appendix~\ref{appx:subsec:auditing_multiple_seqs}). 
Providers could also engage in more sophisticated reordering strategies than the one in Algorithm~\ref{alg:greedy_lookahead}. In response, auditors could test for violations of exchangeability across multiple reported sequences~\citep{vovk2003testingexchangeability, pmlr-v152-vovk21b, dandapanthula2026offline, saha2026distribution}, albeit with less power than the audit considered in our work.

Finally, statistical audits could be complemented by measures that directly attest to faithful execution. For open-weight models, trusted execution environments or zero-knowledge proofs could guarantee users that the provider's computations were carried out faithfully~\citep{9041685, Sun24}. However, such approaches  may be challenging, or even impossible, to implement for proprietary models whose weights and execution details cannot be disclosed.

\xhdr{Reputational and regulatory considerations}
Our results establish the technical feasibility and profiability of inflating the number of reasoning paths charged to a user. 
While reputational or regulatory pressures may deter such behavior, recent controversies over opaque changes to model capabilities suggest that reputational mechanisms alone are insufficient to ensure transparency among LLM service providers.\footnote{\url{https://fortune.com/2026/06/10/anthropic-accu-claude-fable-5-limits-capabilities-ai-researchers-developers/}, \url{https://www.businessinsider.com/researchers-furious-anthropic-mythos-fable-hidden-ai-limits-2026-6}.}


\section{Conclusions}
\label{sec:conclusions}
To what extent can an LLM-as-a-service provider overcharge a user who pays for compute they cannot directly observe?
In our work, we have shown that under self-consistency, the incentive to report an artificially high number of reasoning paths can be substantial. 
Even a provider that commits to a specific stopping rule can generate additional reasoning paths and strategically reorder them so that each appears necessary under the stopping rule---effectively charging the user for computation that was never required. 
Strikingly, an unfaithful provider can carry out this manipulation while guaranteeing that the reported sequence remains statistically indistinguishable from one generated by a faithful provider. As a result, even the most powerful audit cannot reliably detect this deception.
More broadly, our results highlight the risks of opacity in the LLM-as-a-service market and motivate the design of mechanisms that better align providers' incentives with users' interests.

\vspace{2mm}

\xhdr{Acknowledgements} 
Gomez-Rodriguez acknowledges support from the European Research Council (ERC) under the European Union'{}s Horizon 2020 research and innovation programme (grant agreement No. 101169607).

{ 
\small
\bibliographystyle{plainnat}
\bibliography{refs}
}

\clearpage
\newpage

\appendix

\section{Admissibility of Stopping Rules}\label{appx:stopping_rules}

In this section, we show that the stopping rules described in Examples~\ref{ex:ppr-1v1} and~\ref{ex:ASC} of Section \ref{sec:self_consistency} are admissible, \ie, that they satisfy Definition~\ref{def:admissible_count_based_stopping}. We then discuss how the conditions in Definition~\ref{def:admissible_count_based_stopping} apply to other adaptive stopping rules proposed in the literature.

\subsection{Admissibility of PPR-1v1} \label{appx:subsec:ppr-1v1_stopping}

We verify that the PPR-1v1 stopping rule \citep{pac_mode_estimation} satisfies the four conditions of Definition~\ref{def:admissible_count_based_stopping}. Recall that this stopping rule is defined as
\begin{align*}
    \tau_{\text{PPR-1v1}} (s_{(1)}, s_{(2)}) &= \mathbbm{1} \Bigl[ f \Bigl(\frac{1}{2}; s_{(1)}+1, s_{(2)}+1 \Bigr) \leq \frac{\delta}{K-1} \Bigr],
\end{align*}
where $f$ denotes the density of the Beta distribution:
\begin{equation*}
    f \Bigl(\frac{1}{2}; s_{(1)}+1, s_{(2)}+1 \Bigr) = \frac{(1/2)^{s_{(1)}}\cdot \, (1/2)^{s_{(2)}}}{\int_{0}^{1} q^{s_{(1)}} \cdot(1-q)^{s_{(2)}} \, dq} .
\end{equation*}

Thus, a sequence with counts $(\leader, \ru)$ does \emph{not} trigger the stopping rule whenever $f \Bigl(\frac{1}{2}; \leader +1, \ru +1 \Bigr)  > \frac{\delta}{K-1}$.

\begin{itemize}
    \item[i)] Count invariance follows immediately because the stopping decision depends only on the  current leader and runner-up counts, regardless of their identities or the order of preceding samples.
    \item[ii)] Nontriviality, or $\tau_{\text{PPR-1v1}}(1, 0) = 0$, follows from Lemma \ref{lemma:ppr1v1_nontriviality}.
    \item[iii)] Monotonicity in the second most frequent count follows from Lemma \ref{lemma:ppr1v1_runner_up_monotonicity}.
    \item[iv)] Fixed-margin monotonicity follows from Lemma \ref{lemma:ppr1v1_fixed_margin_monotonicity}.
\end{itemize}



\begin{lemma}[Nontriviality of PPR-1v1]
 A single observation is insufficient to trigger PPR-1v1 stopping; that is, $\tau_{\text{PPR-1v1}}(1,0) = 0$.
\label{lemma:ppr1v1_nontriviality}
\end{lemma}

\begin{proof}[Proof of Lemma \ref{lemma:ppr1v1_nontriviality}]
Fix any $\delta \in (0, 1)$ and $K \ge 2$. We want to show  $\tau_{\text{PPR-1v1}}(1,0) = 0$. Evaluating the Beta density at $(s_{(1)},s_{(2)})=(1,0)$  gives:
\begin{align*}
f\Bigl(\frac{1}{2}; 1 + 1, 0 + 1 \Bigr) &=   2^{-(1 + 0)} (0 + 1 +1) \cdot \frac{(1 + 0)!}{1! (1 + 0 - 1)!} = 1.
\end{align*}

Since $\delta \in (0, 1)$ and $K \ge 2$, we have $\frac{\delta}{K-1} \le  \delta < 1$.  Therefore, $f( \frac{1}{2}; 1 + 1, 0 + 1 ) > \frac{\delta}{K-1}$, so the stopping condition is not satisfied and $\tau_{\text{PPR-1v1}}(1,0) = 0$.
\end{proof}

\begin{lemma}[Monotonicity in the second most frequent count of PPR-1v1]
Suppose that counts $(\leader, \ru)$, with $\leader \ge \ru + 1$, do not trigger PPR-1v1. Then $(\leader, \ru + 1)$ also do not trigger PPR-1v1.  
\label{lemma:ppr1v1_runner_up_monotonicity}
\end{lemma}

\begin{proof}[Proof of Lemma \ref{lemma:ppr1v1_runner_up_monotonicity}]

Assume that $\tau_{\text{PPR-1v1}} (\leader, \ru) = 0$. To show $\tau_{\text{PPR-1v1}} (\leader, \ru + 1) = 0$, compare the Beta densities \emph{after} and \emph{before} incrementing the second most frequent count:
\begin{align*}
    \frac{f(\frac{1}{2}; \; \leader + 1, \ru + 2)}{f(\frac{1}{2}; \; \leader + 1, \ru + 1)} &= \frac{(\leader+\ru+2) (\leader+\ru+1)!}{2^{(\leader+\ru+1)} \leader!\,(\ru+1)!} \cdot \frac{2^{(\leader+\ru)}\, \leader! \, \ru!}{(\leader+\ru+1)(\leader+\ru)!} \\
    &= \frac{(\leader+\ru+2) (\leader+\ru+1)!}{2 (\ru +1)!} \cdot \frac{ \ru!}{(\leader+\ru+1)!} \\
    &= \frac{(\leader+\ru+2) \ru! }{2 (\ru+1)!} = \frac{\leader+\ru+2}{2 (\ru+1)}.
\end{align*}

Since $\leader > \ru$, we have $\leader + \ru + 2 > 2(\ru+1)$, meaning the ratio above is greater than one.  Furthermore, since $\tau_{\text{PPR-1v1}}(\leader, \ru)=0$, it follows that $f\bigl( \frac{1}{2}; \leader+1, \ru+1 \big) > \frac{\delta}{K-1}$ and:
\begin{align*}
    f\left( \frac{1}{2}; \leader+1, \ru+2 \right) > f\left( \frac{1}{2}; \leader+1, \ru+1 \right) > \frac{\delta}{K-1}.
\end{align*}
Thus, $\tau_{\text{PPR-1v1}} (\leader, \ru + 1) = 0$.
\end{proof}

\begin{lemma}[Fixed-margin monotonicity of PPR-1v1]
Suppose that counts $(\leader,\ru)$ do not trigger PPR-1v1. Then, holding the margin $\leader - \ru$ fixed, incrementing both counts by a common positive integer also does not trigger PPR-1v1.
\label{lemma:ppr1v1_fixed_margin_monotonicity}
\end{lemma}

\begin{proof}[Proof of Lemma \ref{lemma:ppr1v1_fixed_margin_monotonicity}]
Assume that $\tau_{\text{PPR-1v1}}(\leader, \ru) = 0$ and fix any $c \in \mathbb{N}^+$. 
%
 We first establish the result for $c=1$ by comparing the Beta densities before and after incrementing both counts:
\begin{align*}
  \frac{f \bigl( \frac{1}{2}; \leader +2, \ru+2 \bigr)}{f \bigl( \frac{1}{2}; \leader+1, \ru+1 \bigr)} &= \frac{(\leader+\ru+ 3)!}{2^{2} \, (\leader+1)! \, (\ru+1)!} \cdot \frac{\leader! \, \ru!}{(\leader+\ru+1)!}  \\
  &= \frac{(\leader+\ru+3)(\leader+\ru+2)}{4 (\leader+1)(\ru+1)}. 
\end{align*}

To determine whether this ratio exceeds one, subtract the denominator from the numerator:
\begin{align*}
(\leader+\ru+3)(\leader+\ru+2) - 4 (\leader+1)(\ru+1) &= (\leader - \ru)^2 + \leader + \ru + 2.
\end{align*}

Since $(\leader - \ru)^2 + \leader + \ru + 2 > 0$, we have $f \bigl( \frac{1}{2}; \leader +2, \ru+2 \bigr) > f \bigl( \frac{1}{2}; \leader+1, \ru+1 \bigr)$. 
Iterating the same argument $c$ times gives $f(\frac{1}{2}; \leader + c + 1, \ru + c + 1) >f(\frac{1}{2}; \leader + 1, \ru + 1) >  \frac{\delta}{K-1}$, and therefore $\tau_{\text{PPR-1v1}}(\leader + c, \ru + c) = 0$.

\end{proof}

\subsection{Admissibility of ASC}
We verify that the ASC Beta criterion \citep{aggarwal2023samplestepbystep} with confidence parameter $\gamma \in (3/4, 1)$ satisfies the four conditions of Definition \ref{def:admissible_count_based_stopping}. Recall that the ASC stopping rule is:
\begin{align*}
    \tau_{\text{ASC}}(\leader, \ru) &= \mathbbm{1} \Bigl[\frac{\int_{1/2}^{1} q^{\leader} \cdot (1-q)^{\ru} \, dq}{\int_{0}^{1} q^{\leader}\cdot(1-q)^{\ru} \, dq} \ge \gamma  \Bigr].
\end{align*}

For convenience, define:
\begin{align*}
    H(\leader, \ru) &:= \frac{\int_{1/2}^{1} q^{\leader} \cdot (1-q)^{\ru} \, dq}{\int_{0}^{1} q^{\leader}\cdot(1-q)^{\ru} \, dq}.
\end{align*}

A sequence with counts $(\leader, \ru)$ does \emph{not} trigger the ASC stopping rule whenever $H(\leader, \ru) < \gamma$. For integer counts, the beta-binomial identity gives: 
\begin{align*}
    H(\leader, \ru) &= \text{Pr}\bigl(X_{\leader + \ru + 1} \le \leader \bigr), \qquad \text{ where } X_{\leader + \ru + 1}  \sim \mathrm{Binom}\Bigl(\leader + \ru + 1, \frac{1}{2}\Bigr).
\end{align*}
 
 We use this representation to verify the four admissibility conditions.

\begin{itemize}
    \item[i)] Count invariance follows immediately because $H(\leader, \ru)$, and hence the $\tau_{\text{ASC}}$ stopping decision, depends only on counts $(\leader, \ru)$.
    \item[ii)] Nontriviality, or $\tau_{\text{ASC}}(1, 0) = 0$, follows from Lemma~\ref{lemma:asc_nontriviality}.
    \item[iii)] Monotonicity in the second most frequent count follows from Lemma~\ref{lemma:asc_runner_up_monotonicity}.
    \item[iv)] Fixed-margin monotonicity follows from Lemma~\ref{lemma:asc_fixed_margin_monotonicity}. 
\end{itemize}

\begin{lemma}[Nontriviality of ASC]
For any $\gamma > 3/4$, a single observation is insufficient to trigger the ASC Beta criterion; that is, $\tau_{\text{ASC}}(1,0) = 0$.
\label{lemma:asc_nontriviality}
\end{lemma}

\begin{proof}[Proof of Lemma \ref{lemma:asc_nontriviality}]
For $\leader=1$ and $\ru = 0$, 
\begin{align*}
    H(1, 0) &=  \frac{\int_{1/2}^1 q \; dq}{\int_{0}^{1} q \; dq } = \frac{3}{4}.
\end{align*}
Therefore, whenever $\gamma > \frac{3}{4}$, we have $H(1, 0) < \gamma$, and hence  $\tau_{\text{ASC}}(1,0) = 0$.
\end{proof}

\begin{lemma}[Monotonicity in the second most frequent count of ASC]
Suppose that counts $(\leader, \ru)$, with $\leader \ge \ru + 1$, do not trigger the ASC stopping rule. Then $(\leader, \ru + 1)$ also do not trigger the ASC stopping rule.  
\label{lemma:asc_runner_up_monotonicity}
\end{lemma}

\begin{proof}[Proof of Lemma \ref{lemma:asc_runner_up_monotonicity}]
Assume that $\tau_{\text{ASC}} (\leader, \ru) = 0$, and let $m = \leader + \ru + 1$. Using the beta-binomial representation, $H(\leader, \ru) = \text{Pr}(X_m \le \leader)$. After increasing the runner-up count by one, $H(\leader, \ru + 1) = \text{Pr}(X_{m + 1} \le \leader)$. Write $X_{m+1} = X_{m} + B$, where $B \sim \mathrm{Bernoulli}(\frac{1}{2})$ is independent of $X_m$. Then:
\begin{align*}
    H(\leader, \ru + 1) &= \text{Pr}(X_m + B \le \leader)  \\
    &= \text{Pr} (X_m \le \leader - 1) +  \frac{1}{2} \text{Pr}(X_m = \leader) \\
    &= \text{Pr} (X_m \le \leader ) -  \frac{1}{2} \text{Pr}(X_m = \leader) \\
    &< H(\leader, \ru).
\end{align*}

Thus, increasing the runner-up count strictly decreases the ASC confidence in the empirical leader. From our assumption that $\tau_{\text{ASC}} (\leader, \ru) = 0$, we know $H(\leader, \ru) < \gamma$ and therefore $H(\leader, \ru + 1) < H(\leader, \ru) < \gamma$. This shows  $\tau_{\text{ASC}} (\leader, \ru + 1) = 0$.
\end{proof}

\begin{lemma}[Fixed-margin monotonicity of ASC]
Suppose that counts $(\leader,\ru)$ do not trigger the ASC stopping rule with $\gamma > 3/4$. Then, holding $\leader - \ru$ fixed, incrementing both counts by a common positive integer also does not trigger the ASC stopping rule.
\label{lemma:asc_fixed_margin_monotonicity}
\end{lemma}

\begin{proof}[Proof of Lemma \ref{lemma:asc_fixed_margin_monotonicity}]
Assume that $\tau_{\text{ASC}}(\leader, \ru) = 0$ and fix any $c \in \mathbb{N}^+$.Let $m = \leader + \ru + 1$ and $X_m \sim \mathrm{Binom}(m, \frac{1}{2})$ so that $H(\leader, \ru) = \text{Pr}(X_m \le \leader)$. After increasing both counts by one, we have $H(\leader + 1, \ru + 1) = \text{Pr}(X_{m + 2} \le \leader + 1)$. Write $X_{m + 2} = X_m + B_1 + B_2$, where $B_1, B_2$ are independent $\mathrm{Bernoulli}(\frac{1}{2})$ random variables. Conditioning on $X_m$,
\begin{align*}
    H(\leader + 1, \ru + 1) &= \text{Pr}(X_m \le \leader - 1) + \frac{3}{4} \Pr(X_m = \leader) + \frac{1}{4} \Pr(X_m  = \leader + 1) \\
    &= H(\leader, \ru) + \frac{1}{4} \Bigl[\text{Pr}(X_m = \leader + 1) - \Pr(X_m = \leader) \Bigr].
\end{align*}
For $X_m \sim \mathrm{Binom}(m, \frac{1}{2})$ and $\ru \leq \leader$,
\begin{align*}
\frac{\text{Pr}(X_m = \leader + 1)}{\text{Pr}(X_m = \leader)} &= \frac{m - \leader}{\leader + 1} = \frac{\ru + 1}{\leader + 1} \leq 1.
\end{align*}
Therefore, $\text{Pr}(X_m = \leader + 1) \le \text{Pr}(X_m = \leader)$ which implies $H(\leader + 1, \ru + 1) \le H(\leader, \ru)$. Iterating the same argument $c$ times gives $H(\leader + c, \ru + c) \leq H(\leader, \ru)$ for any $c \in \mathbb{N}^+$. From our assumption that $\tau_{\text{ASC}}(\leader, \ru) = 0$, we have $H(\leader, \ru) < \gamma$. Thus,  $H(\leader + c, \ru + c) \le H(\leader, \ru) < \gamma$. This shows $\tau_{\text{ASC}}(\leader + c, \ru + c) = 0$, as required. 
\end{proof}

\subsection{Additional Self-Consistency Stopping Rules}\label{app:additional_rules}

The admissibility conditions in Definition~\ref{def:admissible_count_based_stopping} characterize stopping rules that depend only on the counts of the empirical mode and runner-up and satisfy natural monotonicity properties. Here, we briefly review other adaptive self-consistency methods from the literature, many of which use information beyond the top-two counts and are therefore not admissible. For example, {Early-Stopping Self-Consistency} (ESC)~\citep{li2024escape} stops when all responses within a fixed recent window agree. Since stopping then depends only on which responses appear in the most recent window rather than on cumulative answer counts, ESC violates count invariance and is not admissible.
Other methods leverage additional information beyond the answer counts. {Reliability-Aware Adaptive Self-Consistency} (ReASC)~\citep{kim2026reasc} augments response counts with the model's confidence in its generated responses, as measured by the token probabilities assigned to each reasoning path. {Reasoning-Aware Self-Consistency} (RASC) ~\citep{wan2025reasoning} instead uses a learned network to extract features from individual reasoning paths and estimate their quality. Similarly, ~\citet{huang2026optimal} augment the top two response counts with an informative prior over query difficulty, estimated from past model generations on similar queries.

Beyond determining the number of samples for an individual query, \citet{feng2026optimal} propose BlendASC, which combines adaptive stopping with a shared inference budget across a \emph{set} of queries and allocates samples according to their relative uncertainty. Whether a particular query receives another sample therefore depends not only on its own response counts but also on the states of the other queries competing for the shared budget. Consequently, BlendASC cannot in general be represented by a fixed per-query rule $\tau(s_{(1)}, s_{(2)})$ and is not admissible. Similarly, {Difficulty-Adaptive Self-Consistency} (DSC)~\citep{wang2025dsc} uses estimates of query difficulty across a set of queries to allocate compute.

\subsection{Extensions Beyond Self-Consistency}\label{app:beyond-majority}

Although our main analysis focuses on self-consistency, Algorithm~\ref{alg:greedy_lookahead} exploits features shared by a broader class of sequential test-time compute methods.

\xhdr{Adaptive best-of-$N$} Best-of-$N$~\citep{ chow2024inference,huang2025is}  generates $N$ independent reasoning paths, scores each one using an auxiliary scoring model, and returns the highest-scoring  path as the final response. In this context, adaptive stopping rules can similarly decide how many reasoning paths should be generated. 
The simplest example is a threshold rule that stops once the provider observes a candidate whose reward, assigned by a reward model $R$, exceeds some target $r^*$:
\begin{align*}
    \tau(y_1, \ldots, y_n) &= \mathbbm{1} \left[ \max_{i \le n} R (y_i) \ge r^* \right]. 
\end{align*}
More sophisticated adaptive best-of-$N$ stopping rules may depend not only on the highest observed reward, but also on the empirical reward distribution, posterior uncertainty, or the expected value of drawing an additional sample ~\citep{wan2025beacon, raman2026adabon}. Nevertheless, as we illustrate empirically in Appendix~\ref{appx:subsec:bon_results}, an unfaithful provider can adapt the logic of Algorithm~\ref{alg:greedy_lookahead} to reorder the sequence of answers, and hence the associated sequence of rewards, so that additional paths appear necessary before the stopping rule is triggered.

\xhdr{Sequential search-then-verify procedures} Search-then-verify methods do not draw independent samples, but maintain a partial solution that is repeatedly expanded with additional candidates, scored by a verifier, and pruned. The procedure stops once the search finds a candidate accepted by the verifier  or further search appears unpromising~\citep{yao2023treeofthoughts, koh2025tree, dalal2026}. 
These search steps are history-dependent, unlike the independent answer draws of self-consistency and best-of-$N$. However, we conjecture that the logic underlying Algorithm~\ref{alg:greedy_lookahead} could be extended to generate additional search steps and reorder so that the resulting compute appears necessary. We leave empirical evaluation of this extension to future work.


\clearpage
\newpage
\section{Deferred Proofs for Section~\ref{sec:protocol}} \label{appx:attack_proofs}

This section establishes the theoretical guarantees underlying Algorithm~\ref{alg:greedy_lookahead}. Appendix~\ref{appx:subsec_correctness_continuation_attack} first proves that every sequence returned by  Algorithm~\ref{alg:greedy_lookahead} is compatible with the provider's stopping rule. Appendix~\ref{appx:subsec:count_margin_at_stopping} then derives two consequences of admissibility that characterize a sequence at faithful stopping. Finally, Appendix~\ref{appx:subsec:lower_bound_gain} derives a lower bound on the number of additional reasoning paths generated by Algorithm~\ref{alg:greedy_lookahead}.

\subsection{Proposition~\ref{prop:one_step_correctness}}\label{appx:subsec_correctness_continuation_attack}

\begin{proof}[Proof of Proposition~\ref{prop:one_step_correctness}]
For each $T \geq N$, let $P(T)$ be the statement that whenever the current sequence $\y^{(T)}$ maintained by Algorithm~\ref{alg:greedy_lookahead} has length $T$, no prefix of length $j < T$ triggers the stopping rule $\tau$, and the current value of $\texttt{LastCompatible}$ is compatible with $\tau$ (Definition~\ref{def:compatible-sequence}). We prove by induction that $P(T)$ holds throughout the execution of Algorithm~\ref{alg:greedy_lookahead}.

\xhdr{Base case} At $T = N$, the sequence $\y^{(N)}$ is the faithfully stopped sequence $\y$, which is compatible with $\tau$ by assumption. Hence, $N$ is the first length at which $\tau$ triggers, so no prefix of length $j < N$ triggers $\tau$. Furthermore, $\texttt{LastCompatible}$ is initialized to $\y = \y^{(N)}$, which is compatible with $\tau$. Thus, $P(N)$ holds.

\xhdr{Inductive step} Suppose $P(t)$ holds for some $t \geq N$, and suppose Algorithm~\ref{alg:greedy_lookahead} continues to a sequence of length $t + 1$. Write $\y^{(t)} = (y_1, \ldots, y_t)$ for the current sequence, and let $y_{\mathrm{new}}$ denote the answer from a newly generated reasoning path. The algorithm constructs a candidate sequence $\y''$ in one of two ways. If $\y^{(t)}$ does not trigger $\tau$, it appends the new answer, giving $\y'' = (y_1, \ldots, y_t, y_{\mathrm{new}})$. If $\y^{(t)}$ triggers $\tau$, the algorithm first sets $\texttt{LastCompatible}$ to $\y^{(t)}$---which is compatible with $\tau$, since it triggers $\tau$ while, by the inductive hypothesis, no shorter prefix does---and then inserts the new answer before the final one, giving $\y'' = (y_1, \ldots, y_{t-1}, y_{\mathrm{new}}, y_t)$. In either case, $\texttt{LastCompatible}$ remains compatible with $\tau$ because it is unchanged and compatible by the inductive hypothesis in the first and updated to the compatible sequence $\y^{(t)}$ in the second. We now consider the two possible outcomes of the iteration.
    \begin{enumerate}
        \item \emph{Termination.} The algorithm returns the current sequence $\y^{(t)}$ when the new answer cannot be deferred, either because $y_{\mathrm{new}} = y_t$ or because the reordered prefix $(y_1, \ldots, y_{t-1}, y_{\mathrm{new}})$ triggers $\tau$. This can occur only when $\y^{(t)}$ triggers $\tau$. By the inductive hypothesis, no shorter prefix does, so $\y^{(t)}$ is compatible with $\tau$. Alternatively, if the audit flags $\y''$, the algorithm returns $\texttt{LastCompatible}$, which is compatible. Thus, in either case, the returned sequence is compatible with $\tau$.
        \item \emph{Continuation.} Suppose the algorithm accepts $\y''$ and continues with a sequence of length $t + 1$. Every prefix of $\y''$ of length at most $t - 1$ coincides with a prefix of $\y^{(t)}$ and hence does not trigger $\tau$ by the inductive hypothesis. The prefix of length $t$ equals $\y^{(t)}$ when the new answer is appended, in which case it does not trigger by construction. When the new answer is reordered, the prefix of length $t$ is $(y_1,\ldots,y_{t-1},y_{\mathrm{new}})$, which does not trigger by the check in Algorithm~\ref{alg:greedy_lookahead}. Therefore, no prefix of $\y''$ of length $j < t + 1$ triggers $\tau$. Since $\texttt{LastCompatible}$ also remains compatible, $P(t + 1)$ holds.
    \end{enumerate}

Whenever the algorithm continues it preserves $P(T)$, and whenever it terminates, it returns a sequence that is already compatible. Thus, the sequence returned by Algorithm~\ref{alg:greedy_lookahead} is compatible with $\tau$.
\end{proof}
\clearpage
\newpage

\subsection{Minimum Count Margin at Stopping}\label{appx:subsec:count_margin_at_stopping}

We next derive two properties of admissible stopping rules that are useful for the lower-bound analysis. First, an admissible stopping rule can only trigger when the most frequent answer count leads the runner-up by at least two. Second, the answer that triggers stopping must be the empirical mode.

\begin{proposition}[Minimum stopping margin]\label{prop:top_two_margin}
Let $\tau$ be an admissible stopping rule satisfying Definition~\ref{def:admissible_count_based_stopping}. If $\tau(s_{(1)}, s_{(2)}) = 1$, then $s_{(1)} - s_{(2)} \geq 2$. Furthermore, if $\y = (y_1, \ldots, y_N)$ is compatible with $\tau$, then $y_N = a^N_{(1)}$, where $a^N_{(1)}, \ldots, a^N_{(K)}$ denote the answers ordered by their counts in $\y$.
\end{proposition}

\begin{proof}[Proof of Proposition~\ref{prop:top_two_margin}]
We first show that $\tau$ cannot trigger when the margin $s_{(1)} - s_{(2)}$ is zero or one. By nontriviality, $\tau(1, 0) = 0$. Fixed-margin monotonicity then implies $\tau(b + 1, b) = 0$ for any $b \geq 0$. Thus, $\tau$ cannot trigger when $s_{(1)} - s_{(2)} = 1$.  Next, consider a tie. Since $\tau(b + 1, b) = 0$ for any $b \geq 0$, monotonicity in the second most frequent count implies $\tau(b + 1, b + 1) = 0$. Thus, the rule cannot trigger when the two most frequent counts are equal either. Since $s_{(1)} \ge s_{(2)}$ by definition, any triggering configuration must satisfy $s_{(1)} - s_{(2)} \ge 2$.

It remains to show that the triggering answer $y_N$ must be the most frequent answer $a^N_{(1)}$. Suppose for contradiction that $y_N \neq a^N_{(1)}$. Removing $y_N$ leaves the count $s_{(1)}$ of $a^N_{(1)}$ unchanged, while the second most frequent count either (i) remains at $s_{(2)}$ or (ii) becomes $s_{(2)} - 1$. In case (i), the preceding prefix of length $N - 1$ already has counts top-two counts $(s_{(1)}, s_{(2)})$. By count invariance, it already triggers $\tau$, contradicting that $N$ is the first stopping time. In case (ii), the preceding prefix has counts $(s_{(1)}, s_{(2)} - 1)$. If $\tau(s_{(1)}, s_{(2)} - 1) = 0$, then monotonicity in the second most frequent count implies $\tau(s_{(1)}, s_{(2)}) = 0$, contradicting the assumption that $\tau(s_{(1)}, s_{(2)}) = 1$. Otherwise, $\tau(s_{(1)}, s_{(2)} - 1) = 1$, which again contradicts that $N$ is the first stopping time. Hence, $y_N = a^N_{(1)}$.
\end{proof}

\clearpage
\newpage


\subsection{Lower Bound on Additional Answers}\label{appx:subsec:lower_bound_gain}

In this section, we lower-bound the expected number of additional answers $\mathbb{E}[N' - N]$ reported by an unfaithful provider  running Algorithm~\ref{alg:greedy_lookahead} beyond the faithful stopping time, under no audit ($\alpha = 0$). In particular, we prove Proposition~\ref{prop:provable_gain_binary}.

\begin{proof}[Proof of Proposition~\ref{prop:provable_gain_binary}]
Fix a sequence $\y = (y_1, \ldots, y_N)$ compatible with an admissible stopping rule $\tau$, and write $a^N_{(1)}, \ldots, a^N_{(K)}$ for the answers ordered by their counts in $\y$, with corresponding model probabilities $p_{(1)}, \ldots, p_{(K)}$. Let $\mathbf{Z} = (Z_1, Z_2, \ldots)$ denote the sequence of i.i.d.\ answers that Algorithm~\ref{alg:greedy_lookahead} draws from the model during continuation. Conditional on $\y$, $N'$ is a deterministic function of $\mathbf{Z}$, and all expectations below are taken over $\mathbf{Z}$ conditional on $\y$.

By Proposition~\ref{prop:top_two_margin}, the triggering answer is $y_N = a^N_{(1)}$. Removing this answer leaves the prefix $(y_1, \ldots, y_{N-1})$, whose two most frequent answers are still $a^N_{(1)}$ and $a^N_{(2)}$. Define their count margin as:
\begin{equation*}
    d := s^{N-1}_{(1)} - s^{N-1}_{(2)} = s^{N}_{(1)} - s^{N}_{(2)} - 1 \ge 1,
\end{equation*}
where the inequality follows from the minimum stopping margin of Proposition~\ref{prop:top_two_margin}.
We begin by expressing the expected number of additional answers $\mathbb{E}[N'-N]$ using the tail-sum formula,
\begin{equation*}
    \mathbb{E}[N' - N] = \sum_{k=1}^\infty \mathbb{P}(N' - N \ge k),
\end{equation*}
and lower-bound each term by restricting attention to executions of Algorithm~\ref{alg:greedy_lookahead} that provably do not terminate.
For each $k \ge 1$, let $\mathcal{E}_k$ be the event that $(Z_1, \ldots, Z_k) \in \{a^N_{(1)}, a^N_{(2)}\}^k$ and
\begin{equation}\label{eq:safe_region}
    0 \le \sum_{t=1}^{j} \big(\mathds{1}\{Z_t = a^N_{(2)}\} - \mathds{1}\{Z_t = a^N_{(1)}\}\big) \le d \quad \text{for all } j \le k.
\end{equation}
Condition~\eqref{eq:safe_region} ensures that  the margin of $a^N_{(1)}$ over $a^N_{(2)}$ never exceeds its value  $d$ from the non-triggering prefix, while simultaneously keeping $a^N_{(1)}$ as the most frequent answer. Thus, on $\mathcal{E}_k$,  $a^N_{(1)}$ remains the most frequent answer throughout the first $k$ continuation draws, and the difference between the counts of the most and second-most frequent answers lies in $[0, d]$ for every $j \le k$.

By monotonicity in the second most frequent count and fixed-margin monotonicity (Definition~\ref{def:admissible_count_based_stopping}), none of these intermediate prefixes triggers $\tau$. Algorithm~\ref{alg:greedy_lookahead} therefore incorporates $Z_1, \ldots, Z_k$ without terminating, implying $\mathcal{E}_k \subseteq \{N' - N \ge k\}$. Consequently, 
\begin{equation*}
    \mathbb{E}[N' - N] \ge \sum_{k=1}^\infty \mathbb{P}(\mathcal{E}_k).
\end{equation*}

To bound the above sum, observe that each sequence $(z_1, \ldots, z_k) \in \{a^N_{(1)}, a^N_{(2)}\}^k$ can be viewed as a path on $\{0, 1, \ldots, d\}$ that starts at $0$, steps right on $a^N_{(2)}$, and steps left on $a^N_{(1)}$. The event $\mathcal{E}_k$ requires the path to stay within $\{0, \ldots, d\}$, which we call a \emph{safe} path. Grouping the probability-weighted safe paths by their endpoint $s$ gives:
\begin{equation}\label{eq:g_bound}
    \mathbb{E}[N' - N] \ge \sum_{s=0}^{d} g(s) - 1,
\end{equation}
where $g(s)$ denotes the total probability mass of all safe paths ending at $s$, with $g(0)$ including the empty path; the $-1$ removes the contribution of this empty path. A safe path ending at $s$ must arise either by extending a safe path ending at $s - 1$ by an $a^N_{(2)}$-step or by extending a safe path ending at $s + 1$ by an $a^N_{(1)}$-step. Hence, $g$ solves the linear system:
\begin{equation}\label{eq:recursion_bound}
    \begin{dcases}
        g(0) = 1 + p_{(1)} \cdot g(1),\\
        g(s) = p_{(2)} \cdot g(s - 1) + p_{(1)} \cdot g(s + 1), & s = 1, \ldots, d - 1,\\
        g(d) = p_{(2)} \cdot g(d - 1).
    \end{dcases}
\end{equation}
The coefficient matrix in Eq.~\eqref{eq:recursion_bound} is Toeplitz and, since $p_{(1)} + p_{(2)} \le 1$, nonsingular. To solve the system,  define a sequence $(\xi_j)_{j \ge -1}$ by $\xi_{-1} = 0$, $\xi_0 = \xi_1 = 1$, and
\begin{equation}\label{eq:auxiliary_recurrence}
    \xi_j = \xi_{j-1} - p_{(1)} p_{(2)} \cdot \xi_{j-2}, \quad j \ge 2.
\end{equation}
We claim that $g(s) = p_{(2)}^{\,s} \cdot \xi_{d-s} / \xi_{d+1}$ solves Eq.~\eqref{eq:recursion_bound}. Rearranging Eq.~\eqref{eq:auxiliary_recurrence} as $\xi_{m} = \xi_{m+1} + p_{(1)} p_{(2)}\, \xi_{m-1}$, we verify each of the three equations in Eq.~\eqref{eq:recursion_bound}. First,
\begin{equation*}
    1 + p_{(1)} g(1) = 1 + p_{(1)} p_{(2)} \frac{\xi_{d-1}}{\xi_{d+1}} = \frac{\xi_{d+1} + p_{(1)} p_{(2)}\, \xi_{d-1}}{\xi_{d+1}} = \frac{\xi_d}{\xi_{d+1}} = g(0).
\end{equation*}
Second, for $1 \le s \le d - 1$,
\begin{equation*}
    p_{(2)} g(s - 1) + p_{(1)} g(s + 1) = p_{(2)}^{\,s} \frac{\xi_{d-s+1} + p_{(1)} p_{(2)}\, \xi_{d-s-1}}{\xi_{d+1}} = p_{(2)}^{\,s} \frac{\xi_{d-s}}{\xi_{d+1}} = g(s).
\end{equation*}
Third, since $\xi_0 = \xi_1$,
\begin{equation*}
    p_{(2)} g(d - 1) = p_{(2)}^{\,d} \frac{\xi_1}{\xi_{d+1}} = p_{(2)}^{\,d} \frac{\xi_0}{\xi_{d+1}} = g(d).
\end{equation*}
Combining Eq.~\eqref{eq:g_bound} with $g(0) \ge 1$ then yields
\begin{equation}\label{eq:xi_bound}
    \mathbb{E}[N' - N] \ge \sum_{s=0}^{d} g(s) - 1 \ge \sum_{s=1}^{d} p_{(2)}^{\,s} \frac{\xi_{d-s}}{\xi_{d+1}}.
\end{equation}
To obtain an explicit expression from Eq.~\eqref{eq:xi_bound}, we solve the recurrence in Eq.~\eqref{eq:auxiliary_recurrence} in closed form, treating the case $p_{(1)} + p_{(2)} = 1$ separately because  it admits a cleaner expression.

\xhdr{Case $\bm{p_{(1)} + p_{(2)} < 1}$} When more than two answers have positive probability, the characteristic polynomial associated with Eq.~\eqref{eq:auxiliary_recurrence} is $x \mapsto x^2 - x + p_{(1)} p_{(2)}$. Its roots are real, since $p_{(1)} + p_{(2)} < 1$ implies $p_{(1)} p_{(2)} < 1/4$, and are given by
\begin{equation*}
    \frac{1 \pm \sqrt{1 - 4 p_{(1)} p_{(2)}}}{2} = \sqrt{p_{(1)} p_{(2)}}\; e^{\mp \operatorname{arccosh}\left(\frac{1}{2\sqrt{p_{(1)} p_{(2)}}}\right)}.
\end{equation*}
The solution to Eq.~\eqref{eq:auxiliary_recurrence} is therefore
\begin{equation*}
    \xi_n = (p_{(1)} p_{(2)})^{n/2}\, \frac{\sinh\!\big((n + 1)\vartheta\big)}{\sinh \vartheta} = (p_{(1)} p_{(2)})^{n/2}\, U_n\!\left(\frac{1}{2\sqrt{p_{(1)} p_{(2)}}}\right), \quad \vartheta := \operatorname{arccosh}\!\left(\frac{1}{2\sqrt{p_{(1)} p_{(2)}}}\right),
\end{equation*}
where $U_n$ is the degree-$n$ Chebyshev polynomial of the second kind. Substituting into Eq.~\eqref{eq:xi_bound},
\begin{equation}\label{eq:bound_chebyshev}
    \mathbb{E}[N' - N] \ge \frac{1}{(p_{(1)} p_{(2)})^{(d+1)/2}\, U_{d+1}\!\left(\frac{1}{2\sqrt{p_{(1)} p_{(2)}}}\right)} \sum_{s=1}^{d} p_{(2)}^{\,s}\, (p_{(1)} p_{(2)})^{(d-s)/2}\, U_{d-s}\!\left(\frac{1}{2\sqrt{p_{(1)} p_{(2)}}}\right).
\end{equation}

\xhdr{Case $\bm{p_{(1)} + p_{(2)} = 1}$} When only two answers have positive probability, the characteristic polynomial factors as $x \mapsto x^2 - (p_{(1)} + p_{(2)}) x + p_{(1)} p_{(2)}$, with distinct roots $p_{(1)}$ and $p_{(2)}$ under the unique-mode assumption $p_{(1)} \neq p_{(2)}$. Hence $\xi_n = (p_{(1)}^{\,n+1} - p_{(2)}^{\,n+1}) / (p_{(1)} - p_{(2)})$, and Eq.~\eqref{eq:xi_bound} becomes:
\begin{align}\label{eq:bound_two_answers}
    \mathbb{E}[N' - N] \ge \sum_{s=1}^{d} p_{(2)}^{\,s}\, \frac{p_{(1)}^{\,d-s+1} - p_{(2)}^{\,d-s+1}}{p_{(1)}^{\,d+2} - p_{(2)}^{\,d+2}} = \frac{p_{(1)} p_{(2)}\,(p_{(1)}^{\,d} - p_{(2)}^{\,d}) - d\,(p_{(1)} - p_{(2)})\, p_{(2)}^{\,d+1}}{(p_{(1)} - p_{(2)})\,(p_{(1)}^{\,d+2} - p_{(2)}^{\,d+2})}.
\end{align}
This proves Proposition~\ref{prop:provable_gain_binary}.
\end{proof}

\xhdr{Tighter generalized bounds} The bounds in Eqs.~\eqref{eq:bound_chebyshev} and~\eqref{eq:bound_two_answers} count only paths whose answers lie in $\{a^N_{(1)}, a^N_{(2)}\}$, discarding the probability mass of every execution of Algorithm~\ref{alg:greedy_lookahead} that samples an answer in $\{a^N_{(3)}, \ldots, a^N_{(K)}\}$. Accounting for these paths yields a tighter lower bound. For each $i > 1$, define the initial margin between the leader and answer $a^N_{(i)}$ in the non-triggering prefix by
\begin{equation*}
    d_i := s^{N-1}_{(1)} - s^{N-1}_{(i)} = \sum_{t=1}^{N-1} \mathds{1}\{y_t = a^N_{(1)}\} - \sum_{t=1}^{N-1} \mathds{1}\{y_t = a^N_{(i)}\}.
\end{equation*}
Let $\mathcal{S} = \prod_{i=2}^{K} \{0, 1, \ldots, d_i\}$ denote the lattice of allowable relative displacements. We say that an execution of Algorithm~\ref{alg:greedy_lookahead} is \emph{safe} if, after each draw, the margin of $a^N_{(1)}$ over every other answer $a^N_{(i)}$ stays in $\{0, \ldots, d_i\}$. 
As before, every such prefix is non-triggering.  Grouping safe paths by their displacement vector $\mathbf{s} = (s_2, \ldots, s_K)$ gives
\begin{equation*}
    \mathbb{E}[N' - N] \ge \sum_{\mathbf{s} \in \mathcal{S}} g(\mathbf{s}) - 1,
\end{equation*}
where $g(\mathbf{s})$ is the total probability mass of safe paths ending at $\mathbf{s}$, with $g(\mathbf{0})$ including the empty path. These masses satisfy:
\begin{equation}\label{eq:general_recur}
    g(\mathbf{s}) = \mathds{1}\{\mathbf{s} = \mathbf{0}\} + p_{(1)}\, g(\mathbf{s} + \mathbf{1})\, \mathds{1}\{\mathbf{s} + \mathbf{1} \in \mathcal{S}\} + \sum_{i=2}^{K} p_{(i)}\, g(\mathbf{s} - \mathbf{e}_i)\, \mathds{1}\{\mathbf{s} - \mathbf{e}_i \in \mathcal{S}\},
\end{equation}
where $\mathbf{0} = (0, \ldots, 0)$, $\mathbf{1} = (1, \ldots, 1)$, and $\mathbf{e}_i$ the standard basis vector in $\mathbb{R}^{K-1}$ associated with answer $a^N_{(i)}$. Unlike the top-two bound, Eq.~\eqref{eq:general_recur} has no closed form in general, but it can be evaluated by solving the finite linear system over $\mathcal{S}$.

\clearpage
\newpage

\clearpage
\newpage
\section{Auditing Reported Sequences via Likelihoods}\label{appx:auditing_proofs}

\subsection{Proposition~\ref{prop:likelihood_correctness}}\label{appx:subsec:likelihood_correctness}

We now show that Algorithm~\ref{alg:likelihood} correctly counts the distinct sampling trajectories that could have produced an observed sequence $\y'$---that is, the distinct sequences $\y$ from which Algorithm~\ref{alg:greedy_lookahead} could have generated $\y'$---and therefore returns the likelihood $\mathbb{P}_{H_1}(\y')$. 

\begin{proof}[Proof of Proposition~\ref{prop:likelihood_correctness}]
Let $\y' = (y'_1, \ldots, y'_{N'})$ be a sequence compatible with $\tau$. For each $1 \le j \le N'$, let $\mathcal{T}_j(\y')$ denote the set of distinct sequences $\y$ such that, during the execution of Algorithm~\ref{alg:greedy_lookahead} on $\y$, the variable $\texttt{LastCompatible}$ can take the value $(y'_1,\dots,y'_j)$. We prove by induction that Algorithm~\ref{alg:likelihood} computes $\texttt{Counts}[j] = |\mathcal{T}_j(\y')|$ for every $1\leq j \leq N'$. It then follows that $\mathbb{P}_{H_1}(\y') = \texttt{Counts}[N'] \cdot p_{y'_{N'}} \prod_{i=1}^{N'} p_{y'_i}$.

\xhdr{Base case} For $j=1$, the reported prefix is $(y'_1)$. There is a single sequence $\y = (y'_1)$ for which the execution of Algorithm~\ref{alg:greedy_lookahead} generates $\texttt{LastCompatible} = (y'_1)$, so $\texttt{Counts}[1] = 1$.

\xhdr{Inductive step} Fix $2 \leq j \leq N'$ and suppose the claim holds for all shorter prefixes. Consider any $\y$ that produces $(y'_1, \ldots, y'_j)$ as $\texttt{LastCompatible}$ during the execution of Algorithm~\ref{alg:greedy_lookahead}. The final answer $y'_j$ must have reached position $j$ after some unique number $L \in \{ 0, \ldots, j-2\}$ of consecutive swaps. If $L=0$, then $y'_j$ was generated directly after the reported prefix of length $j-1$, without being reordered. If $L > 0$, then $y'_j$ was generated after the shorter prefix $(y'_1, \ldots, y'_{j-L-1})$ and subsequently permuted past $L$ generated answers. Such a trajectory is feasible only if, for every $k=1, \ldots, L$, the stopping rule $\tau$ triggers on the answer counts of $(y'_1, \ldots, y'_{j-k-1}, y'_j)$, which is precisely the condition checked by $\texttt{PossibleMultipleSwap}(j, L)$ in Algorithm~\ref{alg:likelihood}. 

For a fixed $L$, the portion of the trajectory preceding the generation of $y'_j$ can be any trajectory producing the prefix $(y'_1, \ldots, y'_{j-L-1})$. By the inductive hypothesis, there are $\texttt{Counts}[j-L-1]$ such trajectories. However, each complete trajectory has a unique value of $L$, so the sets of trajectories corresponding to the different values of $L$ are disjoint. Therefore, $\texttt{Counts}[j] = \sum_{L=0}^{j-2} \mathbbm{1} [\texttt{PossibleMultipleSwap}(j, L)] \cdot \texttt{Counts}[j-L-1]$, which is exactly the update performed in line~\ref{line:DP-update} of Algorithm~\ref{alg:likelihood}. Hence, $\texttt{Counts}[j] = |\mathcal{T}_j(\y')|$.

It remains to convert this trajectory count into a likelihood. Any trajectory that produces the observed sequence $\y'$ contains the $N'$ reported model draws, together with one final unreported draw equal to $y'_{N'}$ that causes Algorithm~\ref{alg:greedy_lookahead} to terminate.  Each such trajectory has probability $p_{y'_{N'}}\prod_{i=1}^{N'} p_{y'_i}$ under $H_1$.  Summing over the $\texttt{Counts}[N'] = |\mathcal{T}_{N'}(\y')|$ trajectories that produce $\y'$ gives $\mathbb{P}_{H_1}(\y')=\texttt{Counts}[N'] \cdot p_{y'_{N'}}\prod_{i=1}^{N'} p_{y'_i}$, as required.
\end{proof}

\subsection{Auditing Multiple Sequences}\label{appx:subsec:auditing_multiple_seqs}

Section~\ref{sec:auditing_individual_seqs} considers an auditor who observes a single reported sequence $\y'$. In practice, however, a provider serves a stream of queries, and an auditor may observe sequences reported in response to many of them---either because a single user submits multiple queries, or because the auditor pools transcripts across multiple users. The likelihood-ratio audit of Eq.~\eqref{eq:likelihood-test} extends naturally to this setting.

Suppose the auditor observes $M$ reported sequences  $\y_1, \y_2, \ldots, \y_M$, corresponding to queries $x_1, x_2, \ldots, x_M$. For each query $j$, define the likelihood ratio:
\begin{align*}
    e_j := \frac{\mathbb{P}_{H_1}(\y_j)}{\mathbb{P}_{H_0}(\y_j )}.
\end{align*}

Assuming that the model's generations are independent across queries conditional on the queries, the joint likelihood ratio of the $M$ reported sequences is:
\begin{align*}
    E_M &:= \prod_{j=1}^M e_j = \frac{\prod_{j=1}^{M} {\mathbb{P}}_{H_1}(\y_j \mid x_j)}{\prod_{j=1}^{M} {\mathbb{P}}_{H_0}(\y_j \mid x_j)}.
\end{align*}

Thus, evidence from multiple reported sequences can be accumulated by multiplying their individual likelihood ratios. Importantly, this product retains the e-value interpretation used in Section~\ref{sec:auditing_individual_seqs}. Under $H_0$, each $e_j$ has conditional expectation one given all previously observed queries and reported sequences. Hence, the process $(E_M)_{M \ge 0}$ is a nonnegative martingale under $H_0$ with $E_0 = 1$. An auditor can reject the null hypothesis whenever $E_M \ge \frac{1}{\alpha}$. By Ville's inequality, the probability that a faithful provider is ever falsely flagged is at most $\alpha$~\citep{ramdas2025evalues}:
\begin{align*}
    \mathbb{P}_{H_0} \left( \sup_{M \ge 1} E_M \ge \frac{1}{\alpha} \right) &\leq \alpha. 
\end{align*}
Thus, the audit remains valid as additional queries are observed, while evidence against a provider that repeatedly  uses Algorithm~\ref{alg:greedy_lookahead} can accumulate across queries. This imposes a stronger constraint a strategic unfaithful provider, who must now ensure that
\begin{align*}
    E_{j-1} \cdot \frac{\mathbb{P}_{H_1} (\y_j)}{\mathbb{P}_{H_0} (\y_j)} < \frac{1}{\alpha}
\end{align*}
for every query $j$ in order to avoid detection. To this end, the provider can modify Algorithm~\ref{alg:greedy_lookahead}  by replacing the single-sequence audit condition with:

\begin{equation*}
\texttt{AuditFlag}(\y, \alpha,\y_1,\dots \y_{j-1}) := \mathbbm{1} \left[
    E_{j-1} \cdot\, \texttt{Counts}[N'] \cdot \,p_{y_{N'}}
    \ge \frac{1}{\alpha}\right]
\end{equation*}
where $\y_1,\dots,\y_{j-1}$ are the $j-1$ previously reported sequences and $E_{j-1} = \prod_{\ell = 1}^{j-1} e_\ell$ is their cumulative likelihood ratio.
\clearpage
\newpage
\section{Computing Likelihoods in Linear Time}\label{appx:sec:likelihoods_linear}

To avoid detection by an auditor, a strategic provider must repeatedly evaluate the likelihood of its reported sequence to ensure that it remains below the audit's rejection threshold (Algorithm~\ref{alg:greedy_lookahead}). However, a direct implementation of Algorithm~\ref{alg:likelihood} is expensive. Each $\texttt{Counts}[j]$ is computed using nested loops over possible swap lengths, and each iteration re-evaluates the stopping rule $\tau$ on a prefix of length $O(N')$, resulting in $\Theta({N'}^3)$ running time overall.

 Algorithm~\ref{alg:likelihood_linear} computes the same likelihood in $O(N'K)$ time, linear in the reported sequence length $N'$ for fixed $K$.
 The speedup follows from the observation that, for a fixed endpoint $j$, if the answer $y'_j$ could have been generated earlier and pushed forward through $L$ consecutive swaps, then it could also have been pushed through any smaller number of swaps. Thus, the feasible swap lengths form a single contiguous interval. This allows us to express the sum defining $\texttt{Counts}[j]$ as a contiguous range sum and to precompute the endpoint of this interval in a single pass.

\xhdr{Collapsing the sum over swap lengths}  Algorithm~\ref{alg:likelihood} computes:
\begin{align*}
    \texttt{Counts}[j] := \sum_{L=0}^{j-2} \mathbbm{1}[\texttt{PossibleMultipleSwap}(j, L)] \cdot \texttt{Counts}[j-L-1],
\end{align*}
where evaluating $\texttt{PossibleMultipleSwap}(j, L)$ requires checking each of the $L$ intermediate swaps. More precisely,
\begin{align*}
    \texttt{PossibleMultipleSwap}(j, L) = \bigwedge_{k=1}^{L} \left[\,\tau\bigl(s_{(1)}, s_{(2)} \mid y'_1, \ldots, y'_{j-k-1}, y'_j\bigr)=1\,\right],
\end{align*}
where $\tau(s_{(1)}, s_{(2)} \mid \cdot )$ denotes $\tau$ applied to the two largest answer counts of the sequence supplied after the conditioning bar.  Once a conjunct fails for some $L$, it fails for all larger $L$. Therefore, for each $j$, the feasible swap lengths $L$ form a contiguous window $\{0, 1, \ldots, L_{\max}(j)\}$, and:
\begin{align*}
    \texttt{Counts}[j] = \sum_{L=0}^{L_{\max}(j)} \texttt{Counts}[j-L-1] = \sum_{i=j-1-L_{\max}(j)}^{j-1} \texttt{Counts}[i].
\end{align*}
We evaluate this range sum in constant time by maintaining the prefix sums $S[j] = \sum_{i=1}^{j} \texttt{Counts}[i]$, with $S[j] = 0$ for $j \le 0$. This reduces each entry to a single subtraction once $L_{\max}(j)$ is known:
\begin{align*}
    \texttt{Counts}[j] = S[j-1] - S\bigl[j-2-L_{\max}(j)\bigr].
\end{align*}
Thus, computing $\texttt{Counts}[j]$ requires only $O(1)$ time once the feasible swap interval has been identified.

\xhdr{Computing $\bm{L_{\max}(j)}$ efficiently} It remains to compute $L_{\max}(j)$ without explicitly checking every intermediate swap, \ie,  looping over $k$. Consider a prefix  $(y'_1, \ldots, y'_m)$ that does not trigger $\tau$. A swap involving $y'_j$ is feasible at this prefix precisely when appending $y'_j$ would trigger $\tau$. For each prefix length $m$, define:
\begin{align*}
    \sigma(m) &= \begin{cases}
        v & \text{if appending answer $v$ to the length-$m$ prefix triggers $\tau$,} \\
        \texttt{None} & \text{otherwise.}
    \end{cases}
\end{align*}
By Proposition~\ref{prop:top_two_margin}, whenever a triggering answer exists, it must be  ithe unique empirical leader of the prefix. Hence, the $k$-th swap condition for endpoint $j$ satisfies $\tau(s_{(1)}, s_{(2)} \mid y'_1, \ldots, y'_{j-k-1}, y'_j) = 1 \iff \sigma(j-k-1) = y'_j$. Therefore, $L_{\max}(j)$ is exactly the number of consecutive prefixes $j-2, j-3, \ldots$ for which appending $y'_j$ would trigger stopping.

To compute this quantity efficiently, let $\texttt{runs}[m]$ denote the length of the maximal consecutive run ending at $m$ for which $\sigma$ takes the same non-\texttt{None} value. Then, 
\begin{align*}
    L_{\max}(j) &= \begin{cases}
        \texttt{runs}[j-2] & \text{if $\sigma(j-2) = y'_j$,} \\
        0 & \text{otherwise.}
    \end{cases}
\end{align*}
Once $\sigma(m)$ and $\texttt{runs}[m]$ have been precomputed, $L_{\max}(j)$ can then be obtained in $O(1)$ time.

\xhdr{Overall running time} All required quantities can be computed in a single pass over the reported sequence. For each prefix length $m$, Algorithm~\ref{alg:likelihood_linear} 
updates the running answer counts, extracts the two most frequent answers and their counts,  tests whether appending the current empirical leader to the current prefix would trigger $\tau$, and updates the corresponding run length. 


Maintaining the answer counts takes constant time per newly observed answer, while identifying the two largest counts and the empirical leader requires $O(K)$ time per prefix. The preprocessing stage therefore requires $O(N'K)$ time. Once this preprocessing is complete, computing $L_{\max}(j)$, $\texttt{Counts}[j]$, and the corresponding prefix sum each requires $O(1)$ time per position. Algorithm~\ref{alg:likelihood_linear} therefore computes the likelihood in $O(N'K)$ time overall.

\begin{algorithm}[!ht]
\caption{It computes the likelihood of an observed sequence under $H_1$ in linear time}
\label{alg:likelihood_linear}
\footnotesize
\begin{algorithmic}[1]
\State \textbf{Input:} Reported compatible sequence $\y' = (y'_1,\dots,y'_{N'})$, stopping rule $\tau$, answer probabilities $p_1,\dots,p_K$.
\State \textbf{Output:} Likelihood $\P_{H_1}(\y')$ under Algorithm~\ref{alg:greedy_lookahead}
\vspace{1mm}
\State $\texttt{Counts} \gets \underbrace{(1,0,\dots,0)}_{N' \text{ elements}}$
\State $S \gets \underbrace{(1,0,\dots,0)}_{N' \text{ elements}}$ \Comment{Prefix sums $S[j]=\sum_{t\le j}\texttt{Counts}[t]$}
\State $\sigma \gets \underbrace{(\texttt{None}, \ldots, \texttt{None})}_{N' \text{ elements}}$ \Comment{$\sigma[j]$ is the answer whose appending at position $j$ triggers $\tau$, else \texttt{None}}
\State $\texttt{runs} \gets \underbrace{(0, \ldots, 0)}_{N' \text{ elements}}$ \Comment{$\texttt{runs}[j]$ is the length of the maximal equal-$\sigma$ run ending at $j$}
\State $s \gets (0, \ldots, 0)$ \Comment{$K$ elements tracking answer counts}
\vspace{1mm}
\State \textcolor{mycommentcolor}{/** Single pass over each prefix to build $\sigma$ and \texttt{runs}: $O(N'K)$ **/}
\For{$j=1,\dots, N'-1$}
    \State $s[y'_{j}] \gets s[y'_{j}] + 1$
    \State $v^\star \gets \arg\max_v s[v]$; \quad $s_{(1)}, s_{(2)} \gets$ two largest entries of $s$
    \If{$s_{(1)} > s_{(2)}$ \textbf{and} $\tau(s_{(1)}, s_{(2)}) = 0$ \textbf{and} $\tau(s_{(1)}+1, s_{(2)}) = 1$}
        \State $\sigma[j] \gets v^\star$ \Comment{Appending $v^\star$ to the length-$j$ prefix triggers $\tau$}
    \EndIf
    \If{$\sigma[j] \neq \texttt{None}$ \textbf{ and } $\sigma[j]=\sigma[j-1]$}
        \State $\texttt{runs}[j] \gets \texttt{runs}[j-1] + 1$ \Comment{Extend the run of prefixes with the same triggering answer}
    \Else
        \State $\texttt{runs}[j] \gets \mathbbm{1}[\sigma[j] \neq \texttt{None}]$
    \EndIf
\EndFor
\vspace{1mm}
\State \textcolor{mycommentcolor}{/** Build $\texttt{Counts}$ using that $\texttt{Counts}[j] = \sum_{L=0}^{L_{\max}} \texttt{Counts}[j{-}L{-}1]$ is a contiguous window: $O(N')$ **/}
\For{$j=2, \ldots, N'$}
    \If{$(j-2) \ge 1$ \textbf{and} $\sigma[j-2] = y'_j$}
        \State $L_{\max} \gets \texttt{runs}[j-2]$
    \Else
        \State $L_{\max} \gets 0$
    \EndIf
    \State $\texttt{Counts}[j] \gets S[j-1] - S[j-2-L_{\max}]$
    \State $S[j] \gets S[j-1] + \texttt{Counts}[j]$
\EndFor
\vspace{1mm}
\State \Return $\texttt{Counts}[N']\cdot p_{y'_{N'}} \cdot \prod_{i=1}^{N'} p_{y'_i}$
\end{algorithmic}
\end{algorithm}
\clearpage
\newpage
\section{Additional Experimental Details}\label{appx:sec:generation-details}
Our experiments in Section~\ref{sec:experiments} rely on the dataset released by~\citet{velasco2026ttcgames}, which contains LLM outputs for multiple test-time compute methods across several benchmark datasets and is publicly available on Hugging Face.\footnote{\url{https://huggingface.co/datasets/Human-Centric-Machine-Learning/strategic-ttc-data}} Specifically, we use model outputs generated under self-consistency and best-of-$N$ across three datasets. We refer the reader to Appendix~D of~\citet{velasco2026ttcgames} for a detailed description of the data-generation procedure.

\xhdr{Datasets}
We use LLM outputs from three datasets: \texttt{GSM8K}, \texttt{GPQA}, and \texttt{AIME}. \texttt{GSM8K}~\citep{cobbe2021training} is a mathematics benchmark consisting of grade-school-level problems, \texttt{GPQA}~\citep{rein2023gpqagraduatelevelgoogleproofqa} is a multiple-choice STEM question-answering benchmark, and \texttt{AIME}~\citep{aime25} is a mathematical reasoning benchmark based on problems from the American Invitational Mathematics Examination. All three datasets provide verifiable ground-truth answers for each query and are publicly available on  Hugging Face.\footnote{\url{https://huggingface.co/datasets/openai/gsm8k}}\textsuperscript{,}\footnote{\url{https://huggingface.co/datasets/Idavidrein/gpqa}}\textsuperscript{,}\footnote{\url{https://huggingface.co/datasets/Maxwell-Jia/AIME_2024}}

\xhdr{Models} 
From the \texttt{Llama} family, we consider \texttt{\seqsplit{Llama-3-8B-Instruct}}, \texttt{\seqsplit{Llama-3.1-8B-Instruct}}, \texttt{\seqsplit{Llama-3.2-1B-Instruct}}, and \texttt{Llama-3.2-3B-Instruct}. From the \texttt{Qwen} family, we consider \texttt{\seqsplit{Qwen-2-0.5B-Instruct}}, \texttt{\seqsplit{Qwen-2-1.5B-Instruct}}, \texttt{\seqsplit{Qwen-2-7B-Instruct}}, \texttt{\seqsplit{Qwen-2.5-3B-Instruct}}, and \texttt{\seqsplit{Qwen-2.5-7B-Instruct}}. We additionally consider three reasoning models distilled from \texttt{DeepSeek-R1}: \texttt{\seqsplit{DeepSeek-R1-Distill-Llama-8B}}, \texttt{\seqsplit{DeepSeek-R1-Distill-Qwen-1.5B}}, and \texttt{\seqsplit{DeepSeek-R1-Distill-Qwen-7B}}. 
For the best-of-$N$ experiments in Appendix~\ref{appx:subsec:bon_results}, we use the \texttt{ArmoRM-Llama3-8B-v0.1} reward model to score the outputs generated by these models.
All the models are publicly available through Hugging Face.

\xhdr{Generation details}
The model outputs were generated using the temperature settings recommended in their official Hugging Face model cards: temperature $0.6$ for the \texttt{Llama} family and temperature $0.7$ for the \texttt{Qwen} family. Neither top-$p$ nor top-$k$ sampling was used. For each query, the dataset contains $128$ outputs for non-reasoning models and $32$ outputs for reasoning models. See~\citet{velasco2026ttcgames} for details regarding prompt formatting.

\xhdr{Licenses}
The \texttt{Llama-3} models and \texttt{ArmoRM-Llama3-8B-v0.1} are licensed under the LLAMA 3 COMMUNITY LICENSE AGREEMENT.\footnote{\url{https://www.llama.com/llama3/license/}}
The \texttt{Llama-3.1} models are licensed under the LLAMA 3.1 COMMUNITY LICENSE AGREEMENT.\footnote{\url{https://www.llama.com/llama3_1/license/}}
The \texttt{Llama-3.2} models are licensed under the LLAMA 3.2 COMMUNITY LICENSE AGREEMENT.\footnote{\url{https://www.llama.com/llama3_2/license/}}
The \texttt{Qwen} models are licensed under the Tongyi Qianwen LICENSE AGREEMENT.\footnote{\url{https://github.com/QwenLM/Qwen/blob/main/Tongyi\%20Qianwen\%20LICENSE\%20AGREEMENT/}}
The \texttt{DeepSeek-R1-Distill-Llama-8B}, \texttt{DeepSeek-R1-Distill-Qwen-1.5B}, and \texttt{DeepSeek-R1-Distill-Qwen-7B} models are licensed under the  MIT License.
 \texttt{GPQA} and \texttt{AIME}  are licensed under Creative Commons Attribution 4.0, and \texttt{GSM8K} is licensed under MIT License.  The \texttt{strategic-ttc-data} dataset used in our experiments is also licensed under the  MIT License.

\clearpage
\newpage

\section{Additional Experimental Results}\label{appx:sec:additional_results}

In this section, we present additional experimental results that complement those in Section~\ref{sec:experiments}.

\subsection{Experimental Results on Synthetic Data}\label{app:synthetic-data}

We first empirically examine how the answer-probability distribution
$(p_1,\ldots,p_K)$ affects the expected number of additional answers
$\mathbb{E}[N'-N]$ generated by Algorithm~\ref{alg:greedy_lookahead}.

\xhdr{Experimental setup}\label{appx:subsec:synthetic_generation}
Motivated by the lower bound in Proposition~\ref{prop:provable_gain_binary}, and as well as prior theoretical results showing that the sample complexity of mode identification depends on the probability between the two most likely answers~\citep{feng2026optimal}, we construct answer-probability distributions $(p_1, \dots, p_K)$ with mode $a_1$  by varying the gap $\Delta = p_1 - p_2$ and spreading the remaining probability mass uniformly across $a_2, \dots, a_K$:
\begin{align*}
p_1 = \frac{1 + (K-1)\Delta}{K}, \qquad p_2 = \ldots = p_K = \frac{1 - \Delta}{K}.
\end{align*}
For $K=2$, we consider $20$ equally spaced values of $\Delta \in [0.1, 0.9]$. For each resulting answer-probability distribution, we sample $1,000$ faithfully stopped sequences $\y = (y_1, \dots, y_N)$ under both the PPR-1v1 stopping rule~\citep{pac_mode_estimation} (Example~\ref{ex:ppr-1v1}) and the ASC stopping rule~\citep{aggarwal2023samplestepbystep} (Example~\ref{ex:ASC}).
For ASC, we set $\gamma = 0.95$ following \citet{aggarwal2023samplestepbystep}; for PPR-1v1, we set $\delta = 0.1$.
Starting from each faithfully stopped sequence $\y$, we run Algorithm~\ref{alg:greedy_lookahead} with $\alpha = 0$, corresponding to no audit, to obtain  $\y' = (y'_1, \dots, y'_{N'})$.

\xhdr{Results}
Figure~\ref{fig:lower_bound_provider_gain} shows the empirical mean number of additional answers, $\E[N' - N]$, as a function of the top-two probability gap $p_{(1)} - p_{(2)}$.
Consistent with Proposition~\ref{prop:provable_gain_binary}, the number of additional answers increases as $p_{(1)} \to p_{(2)}$, exceeding the lower bound by a substantial margin.
%
This shows that Algorithm~\ref{alg:greedy_lookahead}  has greater opportunity to extend a sequence when the LLM’s answer-probability distribution assigns similar probabilities to the two most likely answers.

\begin{figure}[th!]
\includegraphics[width=\textwidth]{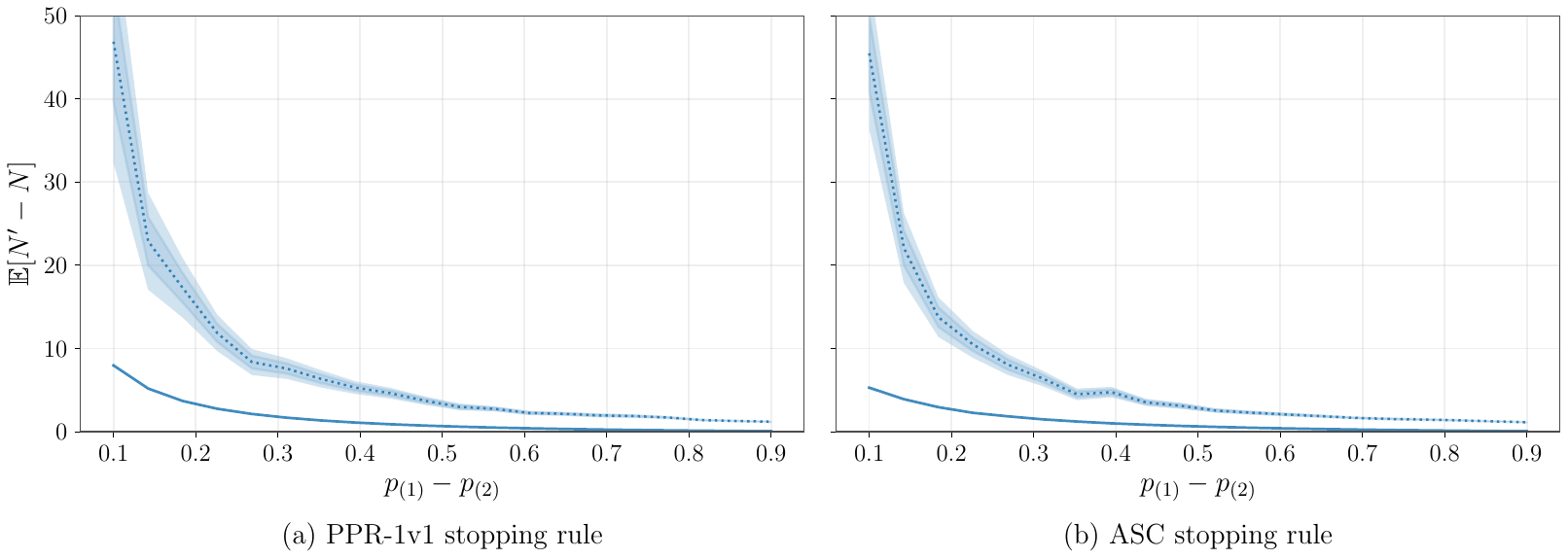}
\caption{\textbf{Additional answers generated by Algorithm~\ref{alg:greedy_lookahead} on synthetic distributions.}
For $K=2$, we report the empirical mean number of additional answers $N' - N$ generated by Algorithm~\ref{alg:greedy_lookahead} (dotted lines), together with the theoretical lower bound from Proposition~\ref{prop:provable_gain_binary} (solid lines), as a function of the top-two probability gap $p_{(1)} - p_{(2)}$.
The shaded regions indicate $\pm 2$ standard errors of the average.
}
\label{fig:lower_bound_provider_gain}
\end{figure}

\subsection{Additional Cost to Users}\label{appx:subsec:additional_llm_results}

 Section~\ref{sec:experiments} measures overcharging in terms of the number of additional reasoning paths generated by Algorithm~\ref{alg:greedy_lookahead}. In practice, however, users are typically billed by the number of tokens the model generates, and reasoning paths can vary substantially in length. This distinction is particularly important for reasoning models, whose outputs may contain long reasoning traces and incur substantially greater token costs per additional sample. As a consequence, even a modest number of additional answers generated by Algorithm~\ref{alg:greedy_lookahead} can result in a substantial price increase for the user. To quantify this effect, Figure~\ref{fig:monetary_overcharge} converts the token counts of the additional reasoning paths generated by Algorithm~\ref{alg:greedy_lookahead} into monetary units and reports the resulting billing overcharge, in U.S. cents, across models and datasets.

\begin{figure}[!ht]
\includegraphics[width=\textwidth]{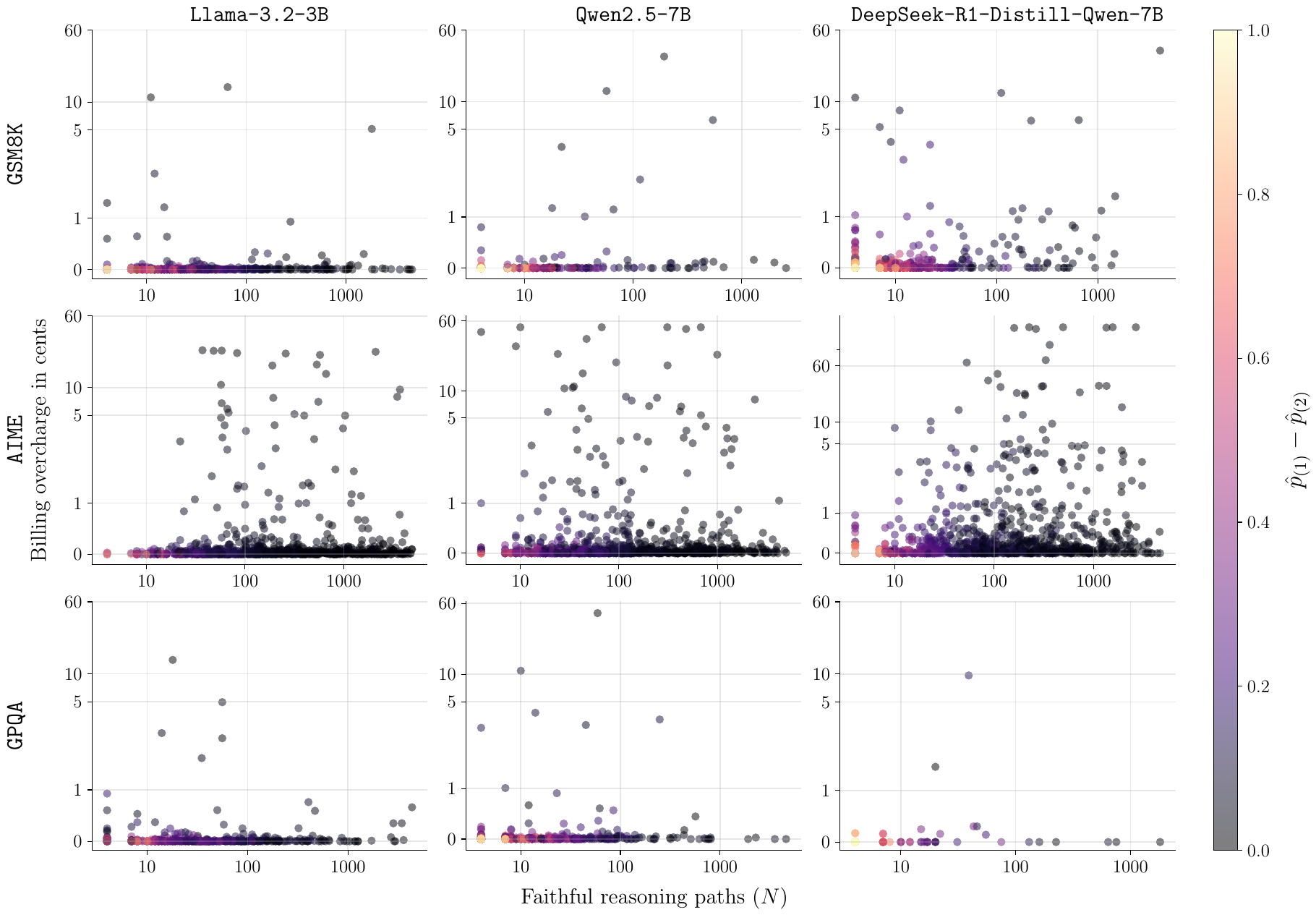}
\caption{\textbf{Price increase from additional answers.} For each query, we plot the additional price, measured in U.S. cents, charged for the $N'-N$ additional reasoning paths generated by Algorithm~\ref{alg:greedy_lookahead}, as a function of the faithful stopping time $N$. Color indicates the estimated top-two probability gap $\hat p_{(1)}-\hat p_{(2)}$,  with smaller gaps corresponding to more difficult queries. We set the audit threshold to $\alpha=0.1$ and use the ASC stopping rule with $\gamma=0.95$.}
\label{fig:monetary_overcharge}
\end{figure}

To compute the monetary overcharge in Figure~\ref{fig:monetary_overcharge}, we convert the observed token counts for the additional answers using hosted inference prices per 1M output tokens from the same platform: $\$0.10$/1M tokens\footnote{\href{https://cloudprice.net/models/meta-llama-3-2-3b-instruct}{https://cloudprice.net/models/meta-llama-3-2-3b-instruct}} for \texttt{Llama-3.2-3B}, $\$0.20$/1M\footnote{\href{https://cloudprice.net/models/alibaba-qwen2-5-7b}{https://cloudprice.net/models/alibaba-qwen2-5-7b}} for \texttt{Qwen2.5-7B}, and $\$0.20$/1M\footnote{\href{https://cloudprice.net/models/deepseek-r1-distill-qwen-7b}{https://cloudprice.net/models/deepseek-r1-distill-qwen-7b}} for \texttt{DeepSeek-R1-Distill-Qwen-7B}. 


\subsection{Experiments with Inadmissible Stopping Rules}\label{appx:subsec:inadmissible_stopping_experiments}

An admissible stopping rule (Definition~\ref{def:admissible_count_based_stopping}) is required only for the theoretical lower bound in Proposition~\ref{prop:provable_gain_binary}. However, the compatibility guarantee in Proposition~\ref{prop:one_step_correctness} does not require the stopping rule to be admissible. Algorithm~\ref{alg:greedy_lookahead} can be applied to construct a longer answer sequence for any stopping rule while still ensuring that the longer sequence is  compatible with the rule. To evaluate whether Algorithm~\ref{alg:greedy_lookahead} remains effective outside the admissible class, we consider Early-Stopping Self-Consistency (ESC)~\citep{li2024escape} with a window size of $5$, which is an inadmissible stopping rule (see Appendix~\ref{app:additional_rules}). Figure~\ref{fig:esc_excess_samples_by_difficulty} shows that Algorithm~\ref{alg:greedy_lookahead}  still generates a substantial number of additional reasoning paths under ESC across models and datasets.

\begin{figure}[!ht]
\includegraphics[width=\textwidth]{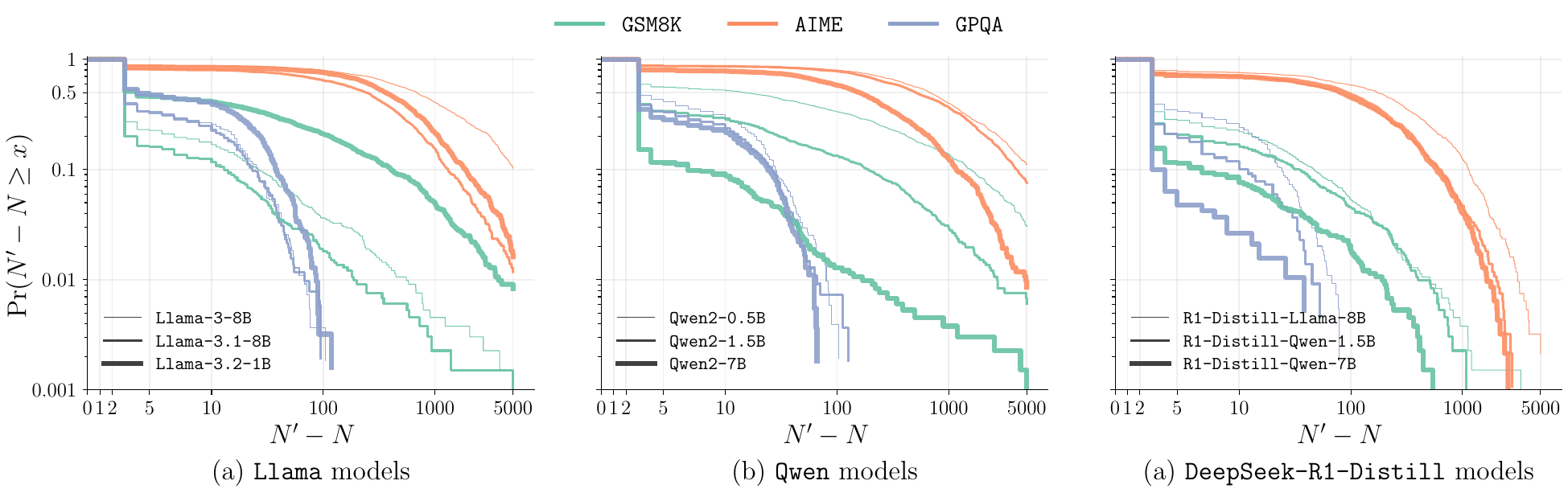}
\caption{\textbf{Distributions of additional reasoning paths under ESC.}  For queries from \texttt{GSM8K}, \texttt{AIME}, and \texttt{GPQA}, each panel shows the complementary cumulative distribution $\Pr(N'-N \ge x)$ of the number of additional reasoning paths generated by Algorithm~\ref{alg:greedy_lookahead} beyond the faithful stopping time. Colors denote datasets, and line width distinguishes model variants within each family.  We use ESC stopping rule  with window size $5$ as the stopping rule and set $\alpha = 0.1$.}
\label{fig:esc_excess_samples_by_difficulty}
\end{figure}

Figure~\ref {fig:esc_auditing} examines the effect of the audit threshold under ESC. Across models and datasets, imposing the audit reduces the provider's ability to extend some sequences, but  substantial overcharging remains at the 90th percentile even for $\alpha=0.1$.
\begin{figure}[!ht]
\centering
\includegraphics[width=\textwidth]{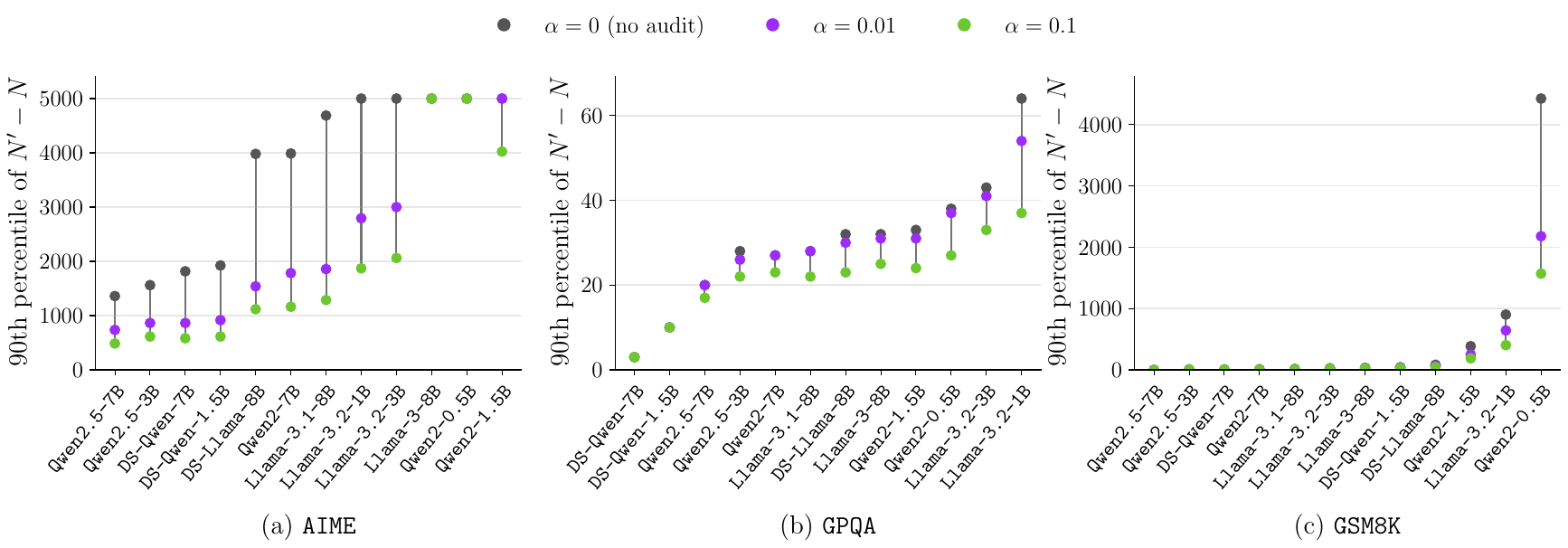}
\caption{\textbf{Influence of the audit threshold under ESC.}
For  queries from \texttt{GSM8K}, \texttt{AIME}, and \texttt{GPQA}, we report the $90$th percentile of the number of additional answers $N'-N$ generated by Algorithm~\ref{alg:greedy_lookahead} using ESC for different audit thresholds $\alpha$. 
The setting $\alpha=0$ corresponds to no audit. 
\texttt{DS} abbreviates \texttt{DeepSeek-R1-Distill}.}
\label{fig:esc_auditing}
\end{figure}

\clearpage
\newpage

\subsection{Adaptive Best-of-N Experiments}\label{appx:subsec:bon_results}

We next show that the conclusions of Section~\ref{sec:experiments} extend beyond self-consistency to adaptive best-of-$N$. We again use the dataset of~\citet{velasco2026ttcgames}, which includes a reward score for each generated answer computed using \texttt{RLHFlow/ArmoRM-Llama3-8B-v0.1}. We consider a simple threshold-based stopping rule that instructs the provider to stop  as soon as it observes a reasoning path with reward at least $r^* = 1.25 \cdot \operatorname{median}(R_1, \ldots, R_{N_{\mathrm{cal}}})$, where $R_1,\ldots, R_{N_{\mathrm{cal}}}$ denote reward scores from a calibration sample. In our experiments, the calibration sample consists of all pre-generated answers ($N_{\mathrm{cal}}=128$ for instruct models and $N_{\mathrm{cal}}=32$ for reasoning models).

Starting from the faithfully stopped sequence, we adapt the logic of Algorithm~\ref{alg:greedy_lookahead} to the reward-based stopping rule, where whenever a reasoning path crosses the reward threshold, the provider generates an additional path and attempts to reorder it so that the reported sequence remains compatible with the stopping rule. Figure~\ref{fig:threshold_bon_distribution} shows that this strategy can generate additional reasoning paths across model families and datasets, suggesting that the opportunity for strategic continuation is not unique to count-based self-consistency.

\begin{figure}[!ht]
\includegraphics[width=\textwidth]{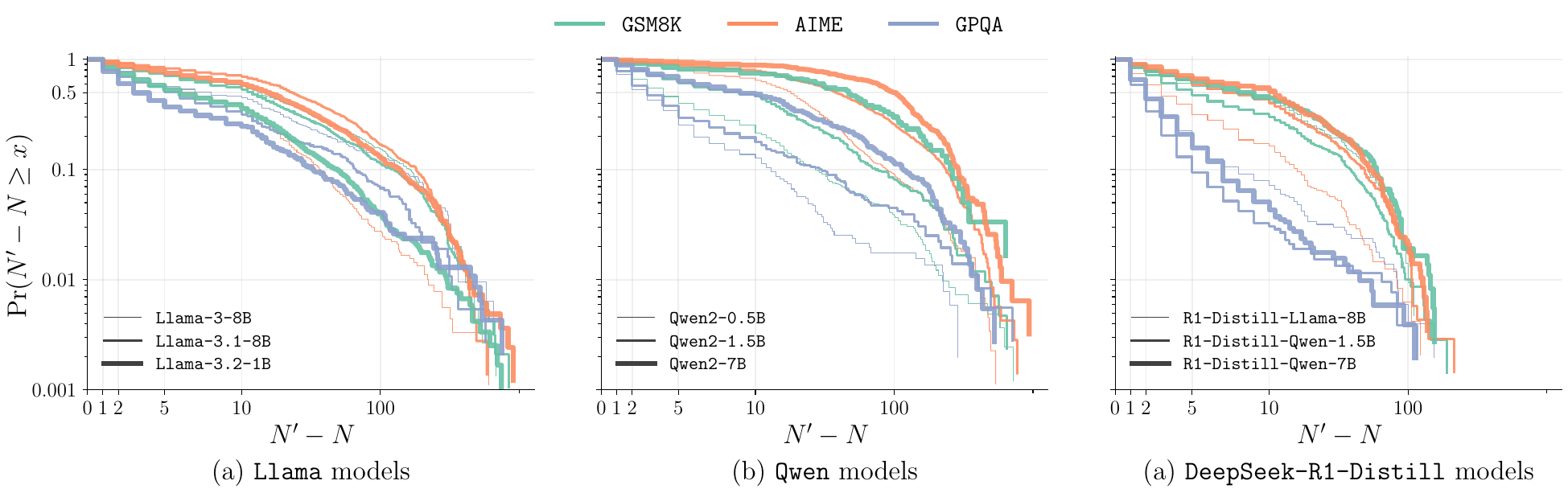}
\caption{\textbf{Distributions of additional reasoning paths under adaptive best-of-$\bm{N}$.}  For queries from \texttt{GSM8K}, \texttt{AIME}, and \texttt{GPQA}, each panel shows the complementary cumulative distribution $\Pr(N'-N \ge x)$ of additional reasoning paths generated beyond the faithful stopping time by adapting Algorithm~\ref{alg:greedy_lookahead} to a reward-threshold stopping rule. Colors denote datasets, and line width distinguishes model variants within each family. These experiments set $\alpha = 0.1$.}
\label{fig:threshold_bon_distribution}
\end{figure}



\xhdr{Auditing adaptive best-of-$\bm{N}$}
Adaptive best-of-$N$ requires a different auditing procedure from the one developed for self-consistency in Section~\ref{sec:auditing_individual_seqs}. While Eq.~\eqref{eq:likelihood-test} evaluates answer sequences under an iid categorical model,  best-of-$N$ stops according to continuous reward scores. An audit must  therefore evaluate whether the observed reward trajectory, including the timing of threshold crossing,  is plausible under faithful execution. A sufficiently informative audit must  distinguish between reward magnitudes and the ordering of reward scores, rather than reducing each reasoning path to whether its reward exceeds $r^*$. This parallels the  likelihood-ratio audit in Section~\ref{sec:auditing_individual_seqs}, which uses the full sequence of answer identities to accumulate evidence against a provider  who reorders answers. Developing analogous audits for adaptive best-of-$N$---for instance, by testing whether the observed reward sequence is consistent with iid draws from the model's reward distribution---is an important direction for future work.


\end{document}